\documentclass{article}

\usepackage{arxiv}

\usepackage[utf8]{inputenc}
\usepackage[T1]{fontenc}
\usepackage{microtype}

\usepackage{amsmath}
\usepackage{amsfonts}
\usepackage{amssymb}
\usepackage{mathtools}
\usepackage{bm}
\usepackage{nicefrac}

\usepackage{booktabs}
\usepackage{array}
\usepackage{graphicx}
\usepackage{float}
\usepackage{multirow}
\usepackage{siunitx}
\usepackage{subcaption}
\usepackage{wrapfig}
\usepackage{arydshln}

\usepackage{algorithm}
\usepackage{algorithmicx}
\usepackage{algpseudocode}

\renewcommand{\algorithmicrequire}{\textbf{Input:}}
\renewcommand{\algorithmicensure}{\textbf{Output:}}
\algrenewcommand\algorithmicrequire{\textbf{Input:}}
\algrenewcommand\algorithmicensure{\textbf{Output:}}

\usepackage{enumitem}
\usepackage{xcolor}
\usepackage{soul}
\usepackage{tcolorbox}

\usepackage{url}
\usepackage{natbib}
\usepackage{hyperref}
\usepackage{doi}

\sethlcolor{purple!20}
\newcommand{\correcthuman}[1]{{\sethlcolor{purple!20}\hl{#1}}}
\newcommand{\humantollm}[1]{{\sethlcolor{orange!25}\hl{#1}}}
\newcommand{\correctllm}[1]{{\sethlcolor{cyan!25}\hl{#1}}}

\usepackage{amsthm}

\makeatletter
\@ifundefined{c@theorem}{
    \newtheorem{theorem}{Theorem}[section]
}{
    \numberwithin{theorem}{section}
}
\makeatother

\newtheorem{assumption}[theorem]{Assumption}
\newtheorem{corollary}[theorem]{Corollary}

\theoremstyle{definition}

\theoremstyle{remark}

\title{Detecting LLM-Generated Tokens in Human--LLM Coauthored Text}

\date{July 23, 2026}	%

\author{
Yangjun Lu\textsuperscript{1},
Hongyi Zhou\textsuperscript{2},
Fabian Spill\textsuperscript{1},
Kai Ye\textsuperscript{3},
Chengchun Shi\textsuperscript{3},
Jin Zhu\textsuperscript{1} \\
\\
\textsuperscript{1}School of Mathematics, University of Birmingham, Birmingham, UK \\
\textsuperscript{2}School of Statistics and Data Science, Shanghai University of Finance and Economics, Shanghai, China \\
\textsuperscript{3}Department of Statistics, The London School of Economics and Political Science, London, UK
}

\hypersetup{
pdftitle={A template for the arxiv style},
pdfsubject={q-bio.NC, q-bio.QM},
pdfauthor={David S.~Hippocampus, Elias D.~Striatum},
pdfkeywords={First keyword, Second keyword, More},
}

\begin{document}
\maketitle
\begin{abstract}
	The rise of human-AI collaborative writing has created a growing need for fine-grained detection methods that support localizing likely LLM-generated content in mixed-authorship documents. Existing methods for detecting LLM-generated text mainly focus on document-level classification and cannot identify which parts of the text are generated by LLMs. This paper introduces a new method to address this urgent need. Our method operates at the token level, the natural unit of modern language models, and builds on existing token-level detection scores. The key idea is to smooth adjacent token scores to reduce their variability, while using an adaptive Lepski-type rule to select the bandwidth according to the local authorship structure. Our method is simple to implement and does not require token-level labeled data for training. Theoretically, we characterize this trade-off and show that the proposed method achieves favorable mean square error performance in estimating the underlying signal. Empirically, we demonstrate strong performance of our method against a wide range of baselines in both synthetic datasets and a realistic dataset. We deploy a publicly accessible website that implements the methods \url{https://huggingface.co/spaces/Jasmineame/detect-LLM-tokens}.
\end{abstract}


\section{Introduction}
Large language model (LLM)-generated text is increasingly used in academic writing, education, news production, and creative writing \citep{milano2023large,demszky2023using,anil2024generativeAI,xie2023next}. In practice, a final document is often neither entirely human-written nor entirely LLM-generated, but typically contains interleaved human-written and LLM-generated segments. For example, an author may draft the main argument while using an LLM to rewrite selected passages, generate descriptions of figures, or continue a story that is later edited by a human. Such human–LLM coauthoring has become increasingly common as LLMs continue to evolve \citep{lee2022coauthor,zhang2024llm}.

In such settings, document-level detection is too coarse: a document-level detector may indicate that a text contains LLM-generated content, but it cannot identify where the LLM contributed or how its contribution is distributed throughout the document. This limitation matters in scientific publishing and editorial review \citep{thorp2023chatgpt,nature2023groundrules}, as well as in legal, policy, and compliance documents where AI usage is worthy to be disclosed \citep{europeanparliament2024aiact}. Localizing AI-generated text can also help narrow the scope of hallucination checking \citep{farquhar2024detecting}. Motivated by these, we study a finer-grained problem: identifying the authorship of individual tokens in a human--LLM coauthored document. Tokens are a natural unit for this task because they are the basic input-output units of modern language models and can be mapped back to localized segments of text. Compared with sentence- or paragraph-level detection \citep{zhang2024machine,su2025haco,li2026segmenting}, token-level detection is better suited to cases where human edits and LLM-generated fragments are interleaved within a sentence, across sentence boundaries, or inside a paragraph. This formulation provides an interpretable way to localize LLM-authored content in general human--LLM collaboration settings \citep{zeng2024aisentence}.

\subsection{Related works}\label{sec:review}

\textbf{LLM-generated document detection}. Existing work on LLM-generated text detection has primarily studied the binary classification problem of distinguishing human-authored document from LLM-authored document. Broadly speaking, these methods fall into four categories. The first category consists of statistics-based detectors, which construct detection scores from token log-probabilities, ranks, likelihood curvature or other related quantities, and then aggregate these quantities over the whole document \citep{gehrmann2019gltr, su2023detectllm, bao2024fastdetectgpt, hans2024spotting, xu2025training}. The second category consists of machine learning (ML)-based detectors. Some methods extract certain feature statistics from text and train classical classifiers such as logistic regression \citep{tulchinskii2023intrinsic, guo2024biscope, mao2024raidar,verma2024ghostbuster,yu2024dpic}; some methods learn adaptive token-level scoring functions before aggregation \citep{chen2025imitate, zhou2025adadetectgpt}; others fine-tunes pretrained language models, such as BERT- or RoBERTa-based models, for text classification \citep{solaiman2019release, ippolito2020bert, hu2023radar, tian2024multiscale,liu2019roberta}. The third category consists of reconstruction- or perturbation-based detectors, which compare the original text with text reconstructed, perturbed, or otherwise transformed by a language model \citep{mitchell2023detectgpt, yang2024dnagpt, zhou2026l2d}. The fourth category is the watermarking-based methods \citep{kirchenbauer2023watermark, zhao2024provable, li2025statistical,xie2025debiasing}. These methods embed detectable statistical signals into the output text, and detect LLM- or human-authored text based on detecting the existence of these signals. 

Although document-level detectors can adapt token-level detection by treating each token as a document, such a crude procedure loses information from the whole text because each token is analyzed independently. An alternative strategy is to apply a document-level detector locally: for each target token, one can evaluate the detector on a short neighborhood around that token and use the resulting score for determining the token's authorship. This local adaptation is attractive because it reuses existing detectors without requiring additional supervised training, but it does not by itself resolve how neighboring tokens should be weighted or how the local bandwidth should be chosen. These issues motivate our proposal developed in this paper.

\textbf{LLM localization under human--AI coauthoring}. In practice, a more meaningful objective is to localize LLM-generated text rather than only to classify the entire document. Existing localization methods can be broadly divided into ML- and segmentation-based approaches. ML-based methods collect mixed-authorship datasets with sentence- or token-level labels and train RoBERTa- or transformer-based models to predict authorship labels \citep{bai2023segformer, wang2023seqxgpt, kushnareva2024boundary, zeng2024towards, zeng2024aisentence, zhang2024machine, li2024spotting, jiang2025sendetex, su2025haco, xu2026unveiling}. However, ML-based methods require labeled data at the sentence- or token-level, which is costly to collect. Segmentation-based methods, by contrast, avoid additional ML model training, by partitioning text into predefined units, most commonly sentences or paragraphs. They then identify LLM-generated regions by selecting transition points between these units. The latter task is achieved by solving a set optimization or change-point detection problem, borrowing techniques from operations research and statistics \citep[see, e.g.][]{lei2025pald, li2026segmenting}. Several human--LLM coauthored datasets have also been introduced to support the evaluation of these localization methods \citep{lee2022coauthor, dugan2023bounary, wang2024semeval}. For segmentation-based methods, their localization precision is inherently limited to the level of the predefined units. These methods implicitly assume that authorship is homogeneous within each predefined unit. Consequently, if an authorship boundary falls within a sentence, such methods cannot localize it. 

To address the above-mentioned limitations, we propose a novel token-level detection method, which applies a document-level detection statistic locally at each token position. These per-token local statistics are then aggregated across neighboring tokens via kernel-weighted smoothing with an adaptive bandwidth. This design not only eliminates the requirement for labeled training data, but also achieves fine-grained token-level localization, without imposing restrictive assumptions on the underlying authorship boundaries.

\subsection{Contributions}
This paper makes the following contributions:
\begin{itemize}[leftmargin=*]
    \item \textbf{Methodology}. We propose a token-level localization procedure that aggregates neighboring token-level detection scores using kernel smoothing, together with an adaptive procedure for selecting the local bandwidth. Unlike fully supervised token-localization approaches, the proposed aggregation step does not require external token-level labeled training data. Unlike segmentation-based approaches, it operates directly at the token level and does not require a predefined sentence- or paragraph-level partition.
    \item \textbf{Theory}. We derive bandwidth-dependent upper bounds that characterize the bias--variance trade-off in estimating local score means. We further analyze the adaptive bandwidth-selection rule and provide an oracle-type guarantee for the smoothed score estimator under stated assumptions.
    \item \textbf{Experiments}. We evaluate the method across four datasets, four language models, and three human--AI collaboration patterns. The results show that our approach can improve token-level ranking performance relative to several baselines on both synthetic and real human--LLM coauthoring datasets. Ablation studies further show that both neighboring-token aggregation and adaptive bandwidth selection contribute positively to performance. 
\end{itemize}
We also release a public online analysis website at \url{https://huggingface.co/spaces/Jasmineame/detect-LLM-tokens}. A case study using the website to analyze an abstract co-written by a human and an LLM shows that the method achieves 94\% accuracy in predicting the authorship of tokens. 

\subsection{Organization}
The rest of the paper is organized as follows. Section 2 introduces the task setting, notation, and existing methods that serve as building blocks for our method. Section 3 introduces the proposed method, and Section 4 presents the theoretical analysis. Section 5 reports numerical studies evaluating the empirical performance of the proposed method. Section 6 illustrates the deployed website with a case study. Section 7 concludes the paper. Technical proofs and implementation details are deferred to the Appendix. 

\section{Preliminaries}\label{sec:preliminary}

\textbf{Task and setting}. We study token-level authorship detection in human--LLM coauthored text. Given a passage $\bm{X}=(X_1, X_2, \ldots, X_L)$, each token $X_t$ has an unobserved authorship label $C_t\in\{0,1\}$, where $C_t=0$ indicates that the token is human-written and $C_t=1$ indicates that it is LLM-generated. The goal is to infer the full label sequence $(C_1,C_2,\ldots,C_L)$ from the observed passage. We write $\bm{X}_{<t}=(X_1,\ldots,X_{t-1})$ for the preceding context, with the convention $\bm{X}_{<1}=\emptyset$.

\textbf{Aggregated token methods for LLM text detection}. We next review a class of document-level detectors built by aggregating token-level statistics. When a passage is generated by an LLM, each token $X_t$ is sampled from a conditional distribution $q_{t}(\cdot \mid \bm{X}_{<t})$ given its preceding context. Motivated by this autoregressive structure, methods such as Fast-DetectGPT, ImBD, and AdaDetectGPT \citep{bao2024fastdetectgpt, chen2025imitate, zhou2025adadetectgpt} compute a token-level score and compare it with a reference expectation. A representative document-level statistic is
\begin{equation}\label{eq:document-stats}
    \sum_{t=1}^{L} \left\{\Phi_t(X_t; \bm{X}_{<t}) - \mathbb{E}_{\widetilde{X}_t \sim q_t(\cdot\mid \bm{X}_{<t})}[\Phi_t(\widetilde{X}_t; \bm{X}_{<t})]\right\},
\end{equation}
where $\Phi_t$ is a token-level score function. The specific choice of $\Phi_t$ differs across methods. Fast-DetectGPT uses $\Phi_t(X_t;\bm{X}_{<t})=\log q_{t}(X_t\mid \bm{X}_{<t})$, the log-probability under a scoring language model. Since raw log-probability can be a weak discriminator, AdaDetectGPT uses $\Phi_t(X_t; \bm{X}_{<t})=\pi_{t}(X_t\mid \bm{X}_{<t})$, where $\pi_{t}$ is learned to better separate LLM texts from human-written ones. ImBD follows a similar principle. Public model checkpoints for AdaDetectGPT and ImBD are available\footnote{The checkpoints of AdaDetectGPT and ImBD are provided at \url{https://huggingface.co/spaces/stats-powered-ai/StatDetectLLM/tree/main} and \url{https://github.com/Jiaqi-Chen-00/ImBD}, respectively.}, so their detection scores can be used directly.

Although these methods are built from token-level quantities, Equation~\eqref{eq:document-stats} aggregates individual tokens over the entire passage and returns a single document-level statistic. It therefore does not directly recover the token-level authorship sequence $(C_1,C_2,\ldots,C_L)$, which is the target of our setting.

\textbf{Lepski's method}. The non-parametric estimation such as Nadaraya--Watson regression \citep{nadaraya1964estimating, watson1964smooth, wasserman2006all} often considers a family of estimators indexed by a bandwidth parameter $k>0$. Smaller bandwidths are sensitive to local changes but have higher variance, whereas larger bandwidths are more stable but can introduce bias through oversmoothing. This results in the bias-variance trade-off.  The ideal bandwidth should balance these two effects, but it is generally unknown in practice.

\begin{figure}[t]
    \centering
    \includegraphics[width=0.45\linewidth]{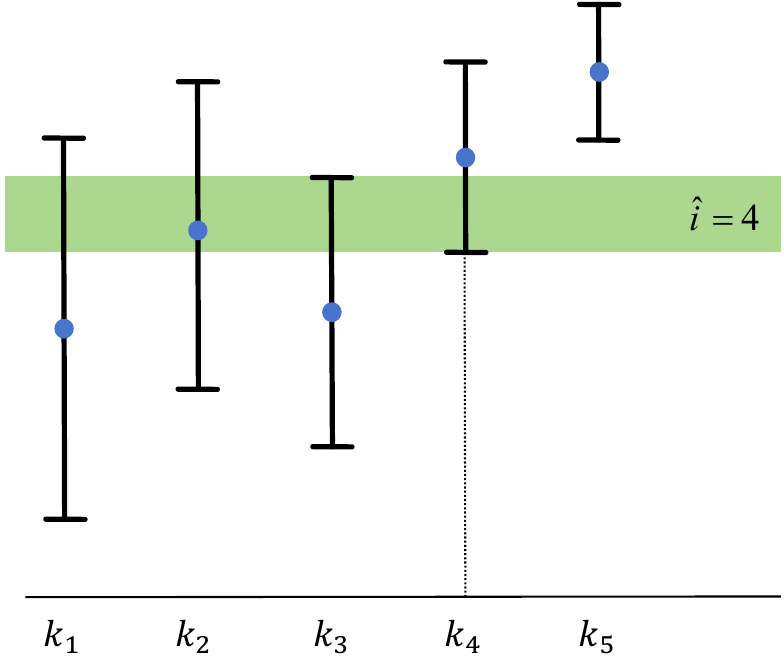}
    \caption{Graphical illustration of Lepski's bandwidth-selection rule over candidate bandwidths $k_1<k_2<\dots<k_5$. Each candidate bandwidth $k_i$ yields a local estimate (blue dot) and a confidence interval (black vertical interval). The green band denotes the intersection of the confidence intervals, starting from the one with the smallest bandwidth and continuing up to the largest accepted bandwidth. Lepski’s rule selects the largest bandwidth for which this intersection is non-empty. Here, the selected bandwidth is $k_4$.}
    \label{fig:lepski}
\end{figure}

\citet{lepski1997optimal} proposed a data-driven method for bandwidth selection. Figure~\ref{fig:lepski} illustrates the rule. For each candidate bandwidth $k_i$, let $I_i$ denote a confidence interval for the corresponding estimator. A candidate bandwidth $k_i$ is compatible with smaller bandwidths if $\bigcap_{j=1}^{i} I_j \neq \emptyset$; for example, $k_4$ is compatible in Figure~\ref{fig:lepski}. Lepski's method then selects the largest compatible bandwidth. The intuition is that smaller bandwidths typically have lower bias but higher variance, and thus serve as low-bias references. As long as the intervals remain compatible, the bias introduced by smoothing is not large enough to make the expectations of these estimators significantly different. Since larger bandwidths yield lower-variance estimators, Lepski's method selects the largest bandwidth that remains compatible with all smaller ones.

\section{Method}\label{sec:method}
This section introduces a kernel-smoothing framework for token-level authorship localization in human--LLM coauthored text. We begin by defining a token-level detection score, where human-written and LLM-generated tokens differ in expectation. Next, we illustrate the procedure that locally aggregates neighboring scores to reduce token-level noise.
 
The document-level statistics reviewed in Section~\ref{sec:preliminary} can be viewed as full-passage aggregations of token-level information. Our proposal instead aggregate the information locally. For each token $X_t$, let $S_t=\Phi_t(\bm{X}_{<t+1})$, where $\Phi_t$ is a predetermined scoring rule that may depend on the token and its context. This token-level detection score covers a broad class of token-level detection statistics described in Section~\ref{sec:preliminary}. Our main assumption is that this score differs, on average, between LLM- and human-authored tokens; we formalize this assumption below.

\begin{assumption}[Score separability]\label{assump:score-separable}
$\mu_1 \neq \mu_0$ where $\mu_{c}=\mathbb{E}[S_t\mid C_t=c]$ for $c\in \{0,1\}$.
\end{assumption}
Assumption~\ref{assump:score-separable} states that the expected token-level detection scores are homogeneous within each source and separated across sources. The homogeneity would like to simplify theoretical analysis, whereas the valid of separation is investigated below. Figure~\ref{fig:box_plot} provides a direct empirical evidence that this average separability is often presented for Fast-DetectGPT- and AdaDetectGPT-type scores on the datasets considered here. It can be seen that the expected scores differ significantly between human- and LLM-authored texts. Additionally, this separation is indirectly supported by the success of document-level detectors \citep{bao2024fastdetectgpt,zhou2025adadetectgpt}. Their statistics are built from averages of token-level scores and have been shown empirically to separate fully human- and LLM-written documents. Without loss of generality, throughout the remainder of the paper, we assume that $\Delta := \mu_1 - \mu_0 > 0$; otherwise, we replace the scoring rule $\Phi_t$ with $-\Phi_t$.

\begin{figure}[t]
  \centering
  \begin{subfigure}[b]{0.48\linewidth}
    \includegraphics[width=\linewidth]{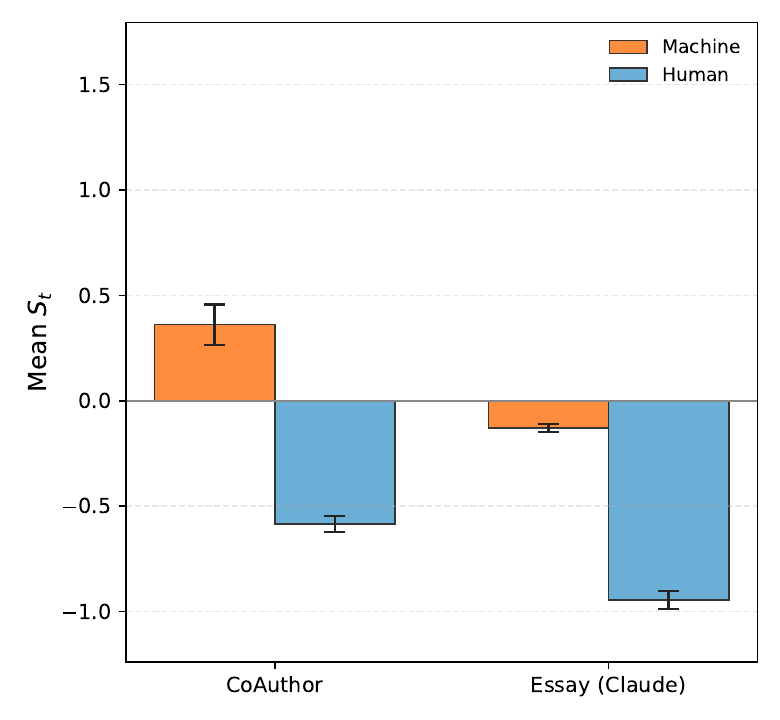}
    \caption{Mean of Fast-DetectGPT-type score.}
    \label{fig:s_t_coauthor}
  \end{subfigure}
  \hfill
  \begin{subfigure}[b]{0.48\linewidth}
    \includegraphics[width=\linewidth]{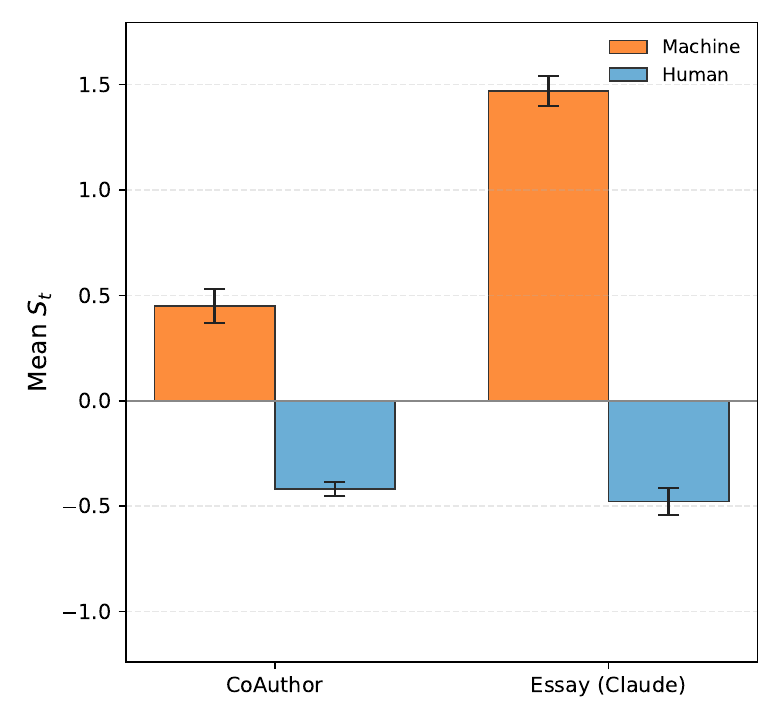}
    \caption{Mean of AdaDetectGPT-type score.}
    \label{fig:s_t_essay}
  \end{subfigure}
  \caption{Mean token-level detection scores for human- and LLM-authored tokens, with error bars denoting 95\% confidence interval. Panel (a) shows the Fast-DetectGPT-type score, and panel (b) shows the AdaDetectGPT-type score. Results are shown on the CoAuthor dataset and on a synthetic dataset described in Section~\ref{sec:compare-baselines}.}
  \label{fig:box_plot}
\end{figure}

To motivate our procedure, we note that a decision rule based directly on the raw score $S_t$ can be unstable because individual token scores are often noisy. To stabilize detection, we observe that authorship labels in coauthored text often persist locally. Any binary label sequence $(C_1,\ldots,C_L)$ can be decomposed into a sequence of constant segments such that there exist indices
\[
1=\nu_0<\nu_1<\cdots<\nu_M<\nu_{M+1}=L+1
\]
and $C_t$ is constant on each segment $[\nu_j,\nu_{j+1}-1]$. Define the minimum segment length as
\begin{equation}\label{eqn:tau}
\tau :=\min_{0\le j\le M}(\nu_{j+1}-\nu_j).
\end{equation}
It is easy to see that $\tau\geq 1$. In realistic coauthored writing, we expect $\tau$ to be much larger than $1$ since both human-written and LLM-generated content typically appears in contiguous units such as phrases, sentences, or paragraphs rather than switching authorship at every token. Motivated by this, we may reduce the variability of raw token scores by locally smoothing them. Specifically, for a bandwidth $k\ge 1$, define the local neighborhood window of token $t$ by
\begin{equation*}
    \textup{Adj}_k(t)=\{l: |l-t|\le k, 1\le l\le L\}.
\end{equation*}
We construct the smoothed score
\begin{equation}\label{sec:token-score}
    \widehat S_{t,k} =  \sum_{l\in \textup{Adj}_k(t)} w_{t,l}^{(k)}S_l,
\end{equation}
where the weights satisfy $w_{t,l}^{(k)}\ge 0$, $\sum\limits_{l\in \textup{Adj}_k(t)} w_{t,l}^{(k)}=1$, and $w_{t,l}^{(k)}=0$ for $l\notin \textup{Adj}_k(t)$. The smoothed version of token score $\widehat S_{t,k}$ can be interpreted as a local estimate of the latent authorship signal around token $t$. In this paper, we focus on a general symmetric kernel-weight smoother with the following form
\begin{equation}\label{sec:token-weight}
    w_{t,l}^{(k)}=\frac{K\left(\frac{|l-t|}{k}\right)}{\sum_{j\in \textup{Adj}_k(t)}
K\left(\frac{|j-t|}{k}\right)},
\qquad l\in \textup{Adj}_k(t),
\end{equation}
where $K(\cdot):[0,1] \to \mathbb{R}_+$ is a nonnegative, non-increasing kernel function. Smoothing reduces variance by averaging nearby scores, but it can introduce bias if the window crosses a human--LLM boundary. In that case, tokens whose labels differ from $C_t$ enter the local average, pulling $\widehat S_{t,k}$ away from the target mean $\mu_{C_t}$. The bandwidth $k$ therefore controls a local bias--variance trade-off: smaller values yield noisier but less contaminated estimates, whereas larger values yield more stable estimates at the cost of a higher risk of crossing an authorship boundary.

Given a smoothed score, we define the token-level classifier as $\widehat C_{t,k}=\mathbb{I}\{\widehat S_{t,k} \ge \eta\}$, where $\eta$ denotes the detection threshold. We choose $\eta$ by minimizing the within-cluster variation:
\begin{equation}\label{eq:clustering}
\arg\min_{\eta} \;\; \sum\limits_{i, j \in \{t \mid \widehat{S}_{t, k} \ge \eta \}} (\widehat{S}_{i, k} - \widehat{S}_{j, k})^2 + \sum\limits_{i, j \in \{t \mid \widehat{S}_{t, k} < \eta \}} (\widehat{S}_{i, k} - \widehat{S}_{j, k})^2.
\end{equation}
Since $\widehat S_{1,k}, \ldots, \widehat S_{L,k}$ are univariate, this optimal solution of the objective~\eqref{eq:clustering} can be found efficiently.

Algorithm~\ref{alg:token_detection} summarizes the complete token-level detection procedure.
\begin{algorithm}[H]
\caption{Token-Level LLM-Generated Text Detection}
\label{alg:token_detection}
\begin{algorithmic}[1]
    \Require Token sequence $\mathbf{X}=(X_1,\ldots,X_L)$, token score function $\Phi_t$, kernel function $K(\cdot)$, and bandwidth $k>0$
    \Ensure Token-level predictions $(\widehat{C}_{1,k},\dots,\widehat{C}_{L,k})$
    \State Compute raw token scores $S_t=\Phi_t(\bm X)$ for $t=1,\ldots,L$
    \For{$t = 1,\ldots,L$}
    \State Compute the smoothed score $\widehat{S}_{t,k}$ using Equations~\eqref{sec:token-score}--\eqref{sec:token-weight}
    \EndFor
    \State Choose a threshold $\eta$ that minimizes~\eqref{eq:clustering}
    \For{$t = 1,\ldots,L$}
    \State $\widehat{C}_{t,k} \leftarrow \mathbb{I}\{\widehat{S}_{t,k}\geq \eta\}$
    \EndFor
\end{algorithmic}
\end{algorithm}
We next discuss three main inputs to Algorithm~\ref{alg:token_detection}.
\begin{itemize}[leftmargin=*]
    \item The score function $\Phi_t$ should satisfy Assumption~\ref{assump:score-separable}; otherwise, the downstream detector has weak signal to classify human- and LLM-authored texts. In practice, stronger document-level token scores often provide stronger token-level localization. Motivated by the empirical performance reported by \citet{zhou2025adadetectgpt}, we use AdaDetectGPT in our main implementation, while noting that the method is compatible with the broader class of scores introduced in Section~\ref{sec:preliminary}.
    \item The kernel function determines how much weight is assigned to neighboring tokens. Theorem~\ref{thm:mse-general} analyzes general kernel weights, and Corollary~\ref{cor:special-kernel} restricts attentions to the uniform and triangular kernels. Since the triangular kernel is more robust to boundary contamination (see the discussion on Corollary~\ref{cor:special-kernel}), we use it as the default kernel in our implementation.
    \item The bandwidth $k$ should be large enough to reduce token-level noise but small enough to avoid excessive contamination across authorship boundaries. %
    We provide a \textit{data-driven} approach for selecting $k$, following \citet{lepski1997optimal}. Let $\{k_1,k_2,\ldots,k_M\}$ be a grid of candidate bandwidths with $k_1<k_2<\cdots<k_M$. For each candidate bandwidth, we compute the score estimator $\widehat{S}_{t,k_i}$ and a radius $\widehat{\mathfrak r}_{t,k_i}$ that quantifies its uncertainty, and then construct the confidence interval
    \[
    I_t(i)=\left[\widehat S_{t,k_i}-2\times\widehat{\mathfrak r}_{t,k_i},\,\widehat S_{t,k_i}+2\times\widehat{\mathfrak r}_{t,k_i}\right]  \textup{ for } i = 1, \ldots, M.
    \]
    We select the largest bandwidth whose interval remains compatible with all smaller-bandwidth intervals:
    \[
    \widehat i_t=\max_{i\in\{1,\ldots,M\}}\left\{i:\bigcap_{j=1}^{i} I_t(j)\ne \emptyset\right\}.
    \]
    The adaptive version of Algorithm~\ref{alg:token_detection} replaces the fixed-bandwidth score $\widehat S_{t,k}$ by $\widehat S_{t,\widehat k_t}$ where $\widehat k_t=k_{\widehat i_t}$.
\end{itemize}

We conclude this section by discussing two practical advantages of our proposal over existing approaches for localizing LLM-generated content in human--LLM coauthored text.
\begin{itemize}[leftmargin=*]
\item Compared with ML-based methods, our method does not require additional token-labeled training data. It instead reuses existing token-level detection scores and converts them into localized authorship predictions.
\item Compared with sentence- or paragraph-level methods, our method does not require a predefined text unit for detection. It operates directly at the token level, making it applicable when authorship changes occur within sentences or other irregular segments.
\end{itemize}

\section{Theory}\label{sec:theory}
This section starts with analyzing the bias--variance trade-off induced by the bandwidth $k$ in the score $\widehat{S}_{t,k}$ (Theorem~\ref{thm:mse-general}). We then instantiate this bound explicitly for uniform and triangular kernels (Corollary~\ref{cor:special-kernel}). Finally, we prove an oracle inequality showing that the data-driven selector for the bandwidth $k$ performs nearly as well as the best fixed bandwidth in the candidate grid (Theorem~\ref{thm:oracle-bound}).

The following theorem establishes an upper bound for the mean squared error of $\widehat S_{t,k}$ as an estimator of $\mu_{C_t}$ to explicitly characterize the bias--variance trade-off. 

\begin{theorem}[MSE of score estimator]\label{thm:mse-general}
Given the label sequence $C=(C_1,\ldots,C_L)$, let 
\[
\sigma^2=\sup_{1\le t\le L}\operatorname{Var}(S_t\mid C) \textup{ and } \rho(r)=\sup_{1\le t\le L-r}\left|\operatorname{Corr}(S_t,S_{t+r}\mid C)\right|
\]
for $r\ge 0$ with the convention that $\rho(0)=1$.
Suppose Assumption \ref{assump:score-separable} holds. Then conditioned on the label sequence $C$, for any token $t$,
\[
\textup{MSE}(\widehat S_{t,k}):= \mathbb E\left(\widehat S_{t,k}-\mu_{C_t}\right)^2
\le\Delta^2\pi_{t,k}^2+\frac{2\kappa\sigma^2}{n_{\mathrm{eff}}(t,k)},
\]
where $\pi_{t,k}=\sum\limits_{l\in \textup{Adj}_k(t):\, C_l\ne C_t}w_{t,l}^{(k)}$ , $n_{\mathrm{eff}}(t,k)=\left(\sum\limits_{l\in \textup{Adj}_k(t)} \left(w_{t,l}^{(k)}\right)^2\right)^{-1}$ and $\kappa = \sum\limits_{r=0}^{L}\rho(r)$.
\end{theorem}

Theorem~\ref{thm:mse-general} gives a direct bias--variance decomposition: 
\begin{itemize}[leftmargin=*]
\item The first term $\Delta^2\pi_{t,k}^2$ is the squared contamination bias induced when the smoothing window crosses an authorship boundary. The quantity $\pi_{t,k}$ measures the total smoothing weight assigned to neighboring tokens whose authorship differs from that of the target token $X_t$. For small value of $k$, the smoothing window contains a substantial mass of tokens from the same source as $X_t$, which keeps the bias term $\pi_{t,k}$ small. In particular, recall that $\tau$ denotes the minimum segment length defined in \eqref{eqn:tau}. When kernel function is non-increasing and $k < \tau$,  it can be shown that the total weight assigned to tokens from the other source is strictly less than one half, i.e., $\pi_{t,k}<1/2$.
\item The second term is the variance term, which is proportional to the inverse effective sample size $(n_{\mathrm{eff}}(t,k))^{-1} = \sum\limits_{l\in \textup{Adj}_k(t)} \left(w_{t,l}^{(k)}\right)^2$. A larger bandwidth usually increases $n_{\mathrm{eff}}(t,k)$, thus reducing $\widehat S_{t,k}$'s variance. For example, under the uniform kernel, $w_{t,l}^{(k)}=\frac{1}{2k+1}$ for $l\in\operatorname{Adj}_k(t)$, so $n_{\operatorname{eff}}(t,k)=2k+1$, which grows linearly with $k$. The factor $\kappa$ can be interpreted as a local dependence measure for the token-score process, indicating that weak dependence of token scores help reduce MSE. Specifically, in weak-dependence settings where token-level correlations decay exponentially with distance, e.g., $\rho(h)\lesssim e^{-ch}$, $\kappa$ is then bounded by a universal constant independent of $k$ and $L$. 
\end{itemize}

Based on Theorem \ref{thm:mse-general}, we immediately have the following corollary.
\begin{corollary}\label{cor:special-kernel}
Let $[a_t,b_t]$ denote the maximal contiguous segment containing token $t$ on which the authorship labels are constant, i.e. $C_j=C_t$ for all $j\in[a_t,b_t]$, and define $d_t=\min\{t-a_t, b_t-t\}$ as the distance from token $t$ to the nearest endpoint of this segment. Suppose the conditions of Theorem~\ref{thm:mse-general} hold and $k<\tau/2$. Then the following bounds hold.
\begin{itemize}[leftmargin=*]
    \item (Uniform Kernel) When $K(u) = \mathbb{I}\{0\le u\le 1\}$, we have 
    \begin{equation*}
        \textup{MSE}(\widehat S_{t,k}) \leq \Delta^2\left[\frac{(k-d_t)_+}{2k+1}\right]^2+\frac{2\kappa\sigma^2}{2k+1},
    \end{equation*}
    where $x_+=\max\{x,0\}$.
    \item (Triangular Kernel) When $K(u) = (1- |u|)_+$, we have
    \begin{equation*}
        \textup{MSE}(\widehat S_{t,k}) \le\Delta^2\left[\frac{(k-d_t)_+(k-d_t+1)_+}{2(k+1)^2}\right]^2
    +2\kappa\sigma^2\frac{2(k+1)^2+1}{3(k+1)^3}.
    \end{equation*}
\end{itemize}
\end{corollary}
Corollary~\ref{cor:special-kernel} derives explicit MSE bounds for two commonly used smoothing kernels. In both cases, the bias term is zero whenever $k\le d_t$, under which condition the smoothing window does not cross the authorship boundary. Once \(k>d_t\), the local window crosses an authorship boundary, and tokens from the other source begin to contaminate the smoothed score. As a result, the bias term increases as bandwidth $k$ becomes larger. Meanwhile, when the dependence factor $\kappa$ is bounded or varies slowly with $k$, the variance term monotonely decreases with $k$ in both cases.

The two kernels differ in their bias/variance bounds. The uniform kernel assigns equal weights to all tokens in the window, yielding a smaller variance term. On the other hand, the triangular kernel is more robust to boundary contamination but at the cost of slightly increasing the variance.

The following theorem shows that with proper choice of the radius estimator $\widehat{\mathfrak r}_{t,k_i}$, the adaptive bandwidth selector is competitive with the oracle bandwidth in the candidate family. 

\begin{theorem}[Adaptive MSE oracle bound]\label{thm:oracle-bound}
Following the same notations as Theorem~\ref{thm:mse-general}, for the $t$-th token, let \(\widehat k_t\) be the bandwidth selected by the Lepski's rule over the candidate bandwidth grid \(\mathcal \{k_1,\ldots,k_M\}\). Suppose that there exists $\xi>0$ such that $\sup\limits_{1\le t\le L}|S_t -\mu_{C_t}|\le \xi$ almost surely, and $\{S_t\}_{t\geq 1}$ are conditionally independent given $C$. When the radius $\widehat{\mathfrak r}_{t,k_i}$ is taken as 
\begin{equation*}
\widehat{\mathfrak r}_{t,k_i} =\left\{2\xi^2\log(2M/\delta)n_{\mathrm{eff}}^{-1}(t,k_i)\right\}^{1/2}
\end{equation*}
for some $\delta>0$, and the kernel function is taken as either uniform kernel or triangular kernel, then the selected bandwidth satisfies
\begin{equation}\label{eqn:radius}
\begin{split}
\mathbb E\left[\left(\widehat S_{t,\widehat k_t}-\mu_{C_t}\right)^2\right]
\le& C_{\operatorname{Lep}}\min_{1\le i\le M}\left[
\Delta^2\left(\max_{1\le j\le i}\pi_{t,k_j}\right)^2
+\frac{\xi^2\log(2M/\delta)}{n_{\mathrm{eff}}(t,k_i)}
\right]
\\
&+4 \left(\xi +\max\{|\mu_0|,|\mu_1|\}\right)^2\delta,
\end{split}
\end{equation}
where $C_{\operatorname{Lep}}$ is a universal constant.
\end{theorem}
Theorem~\ref{thm:oracle-bound} shows that the Lepski-type bandwidth selector performs nearly as well as the best fixed bandwidth in the candidate grid when $\xi$ is known, even though that best bandwidth depends on the unknown local authorship structure. To see why it yields near oracle performance, note that the two terms on the right-hand-side inside the minimum correspond to the contamination bias term and the variance term, respectively. Thus, minimizing over the candidate bandwidths captures the optimal bias--variance trade-off available within the grid. 

To conclude this section, we notice that achieving this oracle bound requires to specify the radius $\widehat{\mathfrak r}_{t,k_i}$ using the scale parameter $\xi$, which may be unknown. In practice, we replace $\xi$ with a variance estimate of the raw token scores and defer the implementation details to Appendix~\ref{app:implementation}.

\section{Experiments}\label{sec:experiments}
This section evaluates the proposed method empirically. Section~\ref{sec:ablation} studies the contribution of local aggregation, kernel choice, and adaptive bandwidth selection through an ablation study. Section~\ref{sec:compare-baselines} and Section~\ref{sec:real} compare our method with existing baselines on synthetic and realistic human--LLM coauthored texts respectively. Section~\ref{sec:sensitivity} examines robustness to authorship-transition patterns and generation temperature. Additional implementation details are provided in Appendix~\ref{app:implementation}.

\subsection{Ablation study}\label{sec:ablation}
Our ablation study investigates the contributions of three components: (i) neighboring-token aggregation, which induces a bias--variance trade-off in local score estimation; (ii) the use of kernel weighting, where we compare the triangular kernel with the uniform kernel, corresponding to unweighted averaging; and (iii) adaptive bandwidth selection via the Lepski rule, which is compared with fixed bandwidths.
 
\textbf{Competing methods}. Toward that end, we compare the following variants:
\begin{itemize}[leftmargin=*]
\item A raw-score baseline that uses each token-level score directly for classification. This is equivalent to our method with $|\textup{Adj}_k(t)| = 1$.
\item Fixed-width kernel smoothing with either the uniform kernel or the triangular kernel. The bandwidth is fixed at $k=7$, so $|\textup{Adj}_k(t)| = 15$.
\item Triangular-kernel smoothing with bandwidth selected adaptively by the Lepski rule.
\item An oracle variant that provides an upper benchmark within the candidate bandwidth set: for each document, it selects $|\textup{Adj}_k(t)|$ from $\{1,15,63,127,255\}$ to maximize performance.
\end{itemize}
For fairness, all variants use the same underlying token-level scoring function.

\textbf{Data generation}. The ablation study uses synthetic human--LLM coauthored texts. We start from a fully human-written text $\bm{X}$ with $N$ sentences. Let $q$ be a positive integer and assume, for simplicity, that $N/(2q)$ is an integer\footnote{If $N/(2q)$ is not an integer, we first partition the sample into $\lfloor N/(2q)\rfloor$ blocks of $2q$ sentences and treat the remaining sentences as a final block containing fewer than $2q$ sentences.}. For each block of $2q$ sentences, we sample an integer $i$ uniformly from $\{1,\ldots,2q\}$, retain the first $i$ sentences as human-written text, prompt an LLM with this prefix to generate the remaining part of the block. The prompt is provided in Appendix~\ref{sec:experiments-details}. Repeating this procedure over all blocks produces a coauthored text $\widetilde{\bm{X}}$ with alternating human-written and LLM-generated segments. This construction gives exact token-level ground-truth labels by design.

We adopt this sentence-level design for three reasons. First, sentences are natural linguistic units, adopting this design makes the data generation process interpretable. Second, because the LLM continuation is generated conditionally on the human-written prefix, the authorship boundary remains fluent, making the setting more realistic than directly concatenating unrelated human and LLM texts. Third, this setting approximates common human--AI writing workflows in which a writer keeps part of a model-generated continuation while modifying or replacing other parts.

\textbf{Text sources and LLM for coauthoring}. Implementing the aforementioned design requires human-written source texts. We use three corpora: Writing \citep{fan2018hierarchical}, XSUM \citep{Narayan2018DonGM}, and ESSAY \citep{verma2024ghostbuster}. For each corpus, we randomly sample 100 human-written texts and apply the generation procedure above to obtain 100 coauthored texts. These corpora cover different writing styles: Writing contains story-style texts, XSUM contains news-style texts, and ESSAY contains education-related texts. The coauthoring LLM is \texttt{google/gemma-2b-it} \citep{team2024gemma}, which is publicly available at \url{https://huggingface.co/google/gemma-2b-it}.

\textbf{Evaluation criterion}. For each text, we compute token-level detection scores and evaluate them against the ground-truth token labels using area under the curve (AUC). We then report the median AUC across documents in each dataset.  

\begin{table}[H]
    \centering
    \small
    \renewcommand{\arraystretch}{1.0}
    \setlength{\tabcolsep}{4pt}
    \caption{Ablation of aggregation methods as $q$ varies from 1 to 8.}
    \label{tab:ablation_adgpt}
    \resizebox{\textwidth}{!}{%
    \begin{tabular}{
        l
        l
        c
        *{8}{S[table-format=1.3]}
    }
    \toprule
    {Dataset} & {Kernel} & {$|\textup{Adj}_k(t)|$}
        & {$q=1$} & {$q=2$} & {$q=3$} & {$q=4$} & {$q=5$} & {$q=6$} & {$q=7$} & {$q=8$} \\
    \midrule

    \multirow{5}{*}{Writing}
        & \multirow{2}{*}{Uniform} & {1}
        & {0.614} & {0.603} & {0.601} & {0.600} & {0.601} & {0.605} & {0.596} & {0.604} \\
        &                         & {15}
        & {0.757} & {0.832} & {0.850} & {0.878} & {0.870} & {0.886} & {0.883} & {0.887} \\
    \cdashline{2-11}[0.4pt/1.5pt]
        & \multirow{3}{*}{Triangular} & {15}
        & {0.726} & {0.855} & {0.873} & {0.908} & {0.900} & {0.925} & {0.909} & {0.928} \\
        &                            & {Auto}
        & {0.720} & {0.848} & {0.871} & {0.918} & {0.912} & {0.939} & {0.919} & {0.948} \\
        &                            & {Oracle}
        & {0.752} & {0.855} & {0.874} & {0.916} & {0.916} & {0.948} & {0.951} & {0.969} \\
    \midrule

    \multirow{5}{*}{XSUM}
        & \multirow{2}{*}{Uniform} & {1}
        & {0.601} & {0.576} & {0.569} & {0.563} & {0.567} & {0.565} & {0.568} & {0.555} \\
        &                         & {15}
        & {0.775} & {0.823} & {0.833} & {0.829} & {0.838} & {0.851} & {0.836} & {0.835} \\
    \cdashline{2-11}[0.4pt/1.5pt]
        & \multirow{3}{*}{Triangular} & {15}
        & {0.771} & {0.856} & {0.867} & {0.868} & {0.867} & {0.889} & {0.879} & {0.875} \\
        &                            & {Auto}
        & {0.761} & {0.854} & {0.878} & {0.879} & {0.878} & {0.898} & {0.895} & {0.889} \\
        &                            & {Oracle}
        & {0.771} & {0.856} & {0.876} & {0.901} & {0.914} & {0.940} & {0.939} & {0.954} \\
    \midrule

    \multirow{5}{*}{ESSAY}
        & \multirow{2}{*}{Uniform} & {1}
        & {0.571} & {0.557} & {0.547} & {0.537} & {0.536} & {0.535} & {0.531} & {0.525} \\
        &                         & {15}
        & {0.712} & {0.736} & {0.741} & {0.735} & {0.728} & {0.735} & {0.719} & {0.706} \\
    \cdashline{2-11}[0.4pt/1.5pt]
        & \multirow{3}{*}{Triangular} & {15}
        & {0.699} & {0.770} & {0.769} & {0.766} & {0.762} & {0.776} & {0.747} & {0.735} \\
        &                            & {Auto}
        & {0.682} & {0.768} & {0.773} & {0.775} & {0.770} & {0.785} & {0.758} & {0.739} \\
        &                            & {Oracle}
        & {0.703} & {0.770} & {0.779} & {0.822} & {0.816} & {0.846} & {0.823} & {0.846} \\
    
    \bottomrule
    \end{tabular}%
    }
\end{table}

The results are reported in Table~\ref{tab:ablation_adgpt}. Across all three datasets and values of $q$, the results show three consistent patterns:
\begin{itemize}[leftmargin=*]
\item Aggregating neighboring token scores consistently improves over the raw token-level score. This confirms the variance-reduction property of local smoothing.
\item The triangular kernel generally outperforms the uniform kernel. This is consistent with the intuition that, near an authorship boundary, the uniform kernel assigns equal weights to all tokens in the local window and therefore places relatively larger weight on tokens from the other source, increasing the contamination bias.
\item The Lepski rule adapts the bandwidth locally and typically improves over a fixed bandwidth. Its performance is also close to the oracle variant, especially on Writing and XSUM, even though the oracle uses label information that is unavailable in practice.
\end{itemize}

\subsection{Comparison with state-of-the-art methods}\label{sec:compare-baselines}

We next compare the proposed method with existing baseline methods on synthetic datasets. The data generation process is the same as in Section~\ref{sec:ablation}, except that we fix $q=2$. To evaluate how each method generalizes beyond the open-source model used in the ablation study, we generate coauthored texts with three closed-source LLMs: Gemini-3.1-flash-lite \citep{google2026gemini31flashlitecard}, Grok-4.1-Fast-non-reasoning \citep{grok2025} (denoted by Grok-4.1-FNR), and Claude-4.5-haiku \citep{anthropic2025claudehaiku45card}.

We consider five baseline methods. The first one is \textit{DeBERTa} \citep{su2025haco}, a token-level ML-based detector trained to predict the author of each token. Because there are few token-level baselines for mixed-authorship detection, we also include two sentence-level methods, %
corresponding to \textit{SegFormer} and \textit{SenPred} \citep{bai2023segformer, kushnareva2024boundary}, which require additional training data to produce sentence-level predictions. Finally, we further include two segmentation-based methods: the ML-based boundary detector \textit{TriBERT} \citep{zeng2024towards} and the topic-based segmentation method \textit{TextTiling} \citep{hearst1997text}. For all ML-based baselines, namely \textit{DeBERTa}, \textit{SegFormer}, \textit{TriBERT}, and \textit{SenPred}, we construct training data from a separate Yelp dataset \citep{zhangCharacterlevelConvolutionalNetworks2015} before evaluating on Writing, XSUM, and ESSAY, in order to avoid data leakage. Dataset construction and training details are deferred to Appendix~\ref{sec:experiments-details}. Our method uses AdaDetectGPT as the token-level scoring function.

As in Section~\ref{sec:ablation}, we use token-level AUC for evaluation. The two token-level methods, ours and DeBERTa, directly produce token-level scores. For sentence-level baselines SegFormer and SenPred, we assign each sentence-level predicted score to all tokens in that sentence and then compute token-level AUC against the ground-truth token labels. TriBERT and TextTiling only partition a text into segments and do not assign authorship labels to the segments. For a fair comparison, we apply the same token-level scoring function to each token, aggregate token scores within each segment according to Equation~\eqref{eq:document-stats}, and assign the resulting segment-level score to all tokens in that segment.
\begin{table}[t]
    \centering
    \small
    \renewcommand{\arraystretch}{1.15}
    \setlength{\tabcolsep}{4pt}
    \caption{Median document-level AUC for three black-box LLMs and three source datasets. The best performance is \textbf{bold}.}
    \label{tab:comparison_table}
    \resizebox{\textwidth}{!}{%
    \begin{tabular}{
        l
        *{9}{S[table-format=1.3]}
    }
    \toprule
    & \multicolumn{3}{c}{Gemini-3.1-flash-lite}
    & \multicolumn{3}{c}{Grok-4.1-FNR}
    & \multicolumn{3}{c}{Claude-4.5-haiku} \\
    \cmidrule(lr){2-4}
    \cmidrule(lr){5-7}
    \cmidrule(lr){8-10}
    {Method}
        & {Writing} & {XSUM} & {ESSAY}
        & {Writing} & {XSUM} & {ESSAY}
        & {Writing} & {XSUM} & {ESSAY} \\
    \midrule
    DeBERTa
        & 0.822 & 0.813 & 0.833
        & 0.812 & 0.814 & 0.790
        & 0.757 & 0.781 & 0.741 \\
    SegFormer
        & 0.793 & \textbf{0.893} & 0.834
        & 0.740 & 0.781 & 0.646
        & 0.713 & 0.762 & 0.616 \\
    SenPred
        & 0.606 & 0.765 & 0.751
        & 0.528 & 0.668 & 0.652
        & 0.573 & 0.628 & 0.692 \\
    TextTiling
        & 0.546 & 0.553 & 0.531
        & 0.552 & 0.559 & 0.523
        & 0.551 & 0.556 & 0.522 \\
    TriBERT
        & 0.706 & 0.716 & 0.711
        & 0.755 & 0.776 & 0.762
        & 0.720 & 0.758 & 0.717 \\
    \hdashline
    Ours
        & \textbf{0.845} & 0.870 & \textbf{0.843}
        & \textbf{0.848} & \textbf{0.882} & \textbf{0.855} 
        & \textbf{0.851} & \textbf{0.887} & \textbf{0.842} \\
    \bottomrule
    \end{tabular}%
    }
\end{table}
The results are summarized in Table~\ref{tab:comparison_table}. Our method achieves the highest AUC in eight of the nine settings and remains competitive in the remaining setting. Compared with ML-based baselines, it is more stable across generators and source datasets. Compared with TextTiling, a training-free baseline, our method achieves more accurate authorship classification. This improvement is expected, as TextTiling segments text based on topical coherence rather than source.

\subsection{Real-data analysis}\label{sec:real}

We also evaluate the proposed method on a real-world CoAuthor dataset \citep{lee2022coauthor}, which was collected under a realistic human--LLM collaborative writing setting. When collecting this dataset, volunteers wrote through a web interface that began with a system-provided prompt. During writing, GPT-3 provided sentence-level suggestions, which writers could accept, reject, or revise before incorporating them into the final document. The resulting documents thus contain both human- and LLM-authored tokens. The dataset contains 1,445 documents written by 63 volunteers, including creative stories and argumentative essays.

The raw CoAuthor data provide character-level authorship labels, which we convert to token-level labels for evaluation. In this conversion, we first derive word-level labels as follows: (i) the initial prompt is treated as human-written; (ii) if the first character of a word is human-written, the entire word is labeled as human-written, reflecting the common case in which a writer begins a word and then completes or edits it with model assistance; and (iii) otherwise, the word label is determined by majority vote over its character labels. We next map word-level labels to token-level labels. If a word corresponds to one token, the token receives the word label. If a tokenizer splits a word into multiple tokens, all resulting tokens receive the word label.

For comparison, we use the same baselines as in Table~\ref{tab:comparison_table}. We additionally include PaLD \citep{lei2025pald}, a training-free text segmentation method based on combinatorial optimization\footnote{PaLD is not considered in Section~\ref{sec:compare-baselines} because of its high computational cost.}. We evaluate all methods using the same token-level AUC protocol as above. For methods that output sentence- or segment-level predictions, we assign each prediction to all tokens in the corresponding sentence or segment before computing the AUC.

Table~\ref{tab:Comparison_in_realistic_dataset} shows that our method outperforms all baselines on the realistic coauthoring dataset. The margin over ML-based baselines such as DeBERTa and SegFormer is larger than that observed on the synthetic datasets, suggesting that ML-based methods may be sensitive to distribution shift between synthetic training data and realistic human--AI coauthoring text. Our method also outperforms sentence- or segment-level baselines such as SenPred and PaLD. These methods implicitly assume that an entire sentence or segment has a single author, an assumption that need not hold in CoAuthor, where model-assisted edits can occur within a sentence.

\begin{table}[t]
    \centering
    \renewcommand{\arraystretch}{1.15}
    \setlength{\tabcolsep}{3pt}
    \caption{Median document-level AUC on the real-world CoAuthor dataset.}
    \label{tab:Comparison_in_realistic_dataset}
    \begin{tabular}{c c c c c c c}
    \toprule
        {DeBERTa}
        & {PaLD}
        & {SegFormer}
        & {SenPred}
        & {TriBERT}
        & {TextTiling}
        & Ours \\
    \midrule
        0.491
        & 0.668
        & 0.521
        & 0.685
        & 0.703
        & 0.601
        & \textbf{0.765} \\
    \bottomrule
    \end{tabular}%
\end{table}

\subsection{Sensitivity analysis}\label{sec:sensitivity}
This section examines how robust different methods are to changes in human–LLM authorship transition patterns and the LLM generation temperature.

\textbf{Robustness to authorship-transition patterns}. We first evaluate robustness under a more irregular switching scheme. Rather than dividing each passage into fixed-size blocks, we allow the source of the text to change at arbitrary sentence boundaries. For each document, we sample the number of transitions $b$ uniformly from $\{1,\cdots,N\}$, and then sample $b$ boundary locations uniformly without replacement from $\{1,\cdots,N\}$. These boundaries divide the passage into variable-length segments. The first segment is human-written, and the authorship then alternates across segments between LLM-generated and human-written text. Each LLM-generated segment is generated based on the human-written segment immediately before it.  We compare against the same set of baseline methods used in Table~\ref{tab:comparison_table}.

Table~\ref{tab:random_transition} shows that our method achieves the best AUC in seven of the nine settings. These results demonstrate the robustness of our proposal to different authorship-transition patterns.

\begin{table}[t]
    \centering
    \small
    \renewcommand{\arraystretch}{1.15}
    \setlength{\tabcolsep}{4pt}
    \caption{Median document-level AUC under the irregular authorship-transition pattern described in Section~\ref{sec:sensitivity}.}
    \label{tab:random_transition}
    \resizebox{\textwidth}{!}{%
    \begin{tabular}{
        l
        *{9}{S[table-format=1.3]}
    }
    \toprule
    & \multicolumn{3}{c}{Gemini-3.1-flash-lite}
    & \multicolumn{3}{c}{Grok-4.1-FNR}
    & \multicolumn{3}{c}{Claude-4.5-haiku} \\
    \cmidrule(lr){2-4}
    \cmidrule(lr){5-7}
    \cmidrule(lr){8-10}
    {Method}
        & {Writing} & {XSUM} & {ESSAY}
        & {Writing} & {XSUM} & {ESSAY}
        & {Writing} & {XSUM} & {ESSAY} \\
    \midrule
DeBERTa
    & 0.799 & 0.855 & {\bfseries 0.847}
    & 0.811 & 0.842 & 0.807
    & 0.741 & 0.807 & 0.797 \\
SegFormer
    & 0.765 & {\bfseries 0.866} & 0.806
    & 0.754 & 0.816 & 0.639
    & 0.657 & 0.715 & 0.587 \\
SenPred
    & 0.586 & 0.727 & 0.704
    & 0.518 & 0.660 & 0.602
    & 0.599 & 0.611 & 0.655 \\
TextTiling
    & 0.609 & 0.664 & 0.631
    & 0.656 & 0.650 & 0.610
    & 0.631 & 0.652 & 0.624 \\
TriBERT
    & 0.748 & 0.746 & 0.721
    & 0.747 & 0.751 & 0.742
    & 0.703 & 0.742 & 0.731 \\
Ours
    & {\bfseries 0.826} & 0.831 & 0.813
    & {\bfseries 0.825} & {\bfseries 0.849} & {\bfseries 0.838}
    & {\bfseries 0.798} & {\bfseries 0.837} & {\bfseries 0.800} \\
    \bottomrule
    \end{tabular}}
\end{table}

\textbf{Robustness across temperatures}. Temperature controls the sampling behavior of language models and can therefore affect the generated synthetic data. To assess whether our conclusions depend on this parameter, we follow the data generation process in Section~\ref{sec:compare-baselines} with temperatures $T\in\{0.3,0.6,0.9\}$. We then evaluate all methods on the resulting datasets. Table~\ref{tab:Temperature_Analysis} shows that our method remains robust across temperatures. It achieves the best performance on XSUM and ESSAY at all three temperatures, achieves the best performance on Writing at $T=0.6$ and $T=0.9$, and remains close to the best method on Writing at $T=0.3$. Overall, the relative ranking of different methods remains largely stable across temperatures, suggesting that our main conclusions are not sensitive to the choice of sampling temperature.
\begin{table}[t]
    \centering
    \small
    \renewcommand{\arraystretch}{1.15}
    \setlength{\tabcolsep}{4pt}
    \caption{Sensitivity to generation temperature for datasets generated by Grok-4.1-FNR.}
    \label{tab:Temperature_Analysis}
    \resizebox{\textwidth}{!}{%
    \begin{tabular}{
        l
        c c c
        c c c
        c c c
    }
    \toprule
    Temperature
    & \multicolumn{3}{c}{0.3}
    & \multicolumn{3}{c}{0.6}
    & \multicolumn{3}{c}{0.9} \\
    \cmidrule(lr){2-4}
    \cmidrule(lr){5-7}
    \cmidrule(lr){8-10}
    {Method}
        & {Writing} & {XSUM} & {ESSAY}
        & {Writing} & {XSUM} & {ESSAY}
        & {Writing} & {XSUM} & {ESSAY} \\
    \midrule
    DeBERTa
    & \textbf{0.837} & {0.852} & {0.824}
    & {0.832} & \textbf{0.879} & {0.808}
    & {0.838} & {0.844} & {0.825} \\
    SegFormer
    & 0.740 & 0.790 & 0.637
    & 0.741 & 0.794 & 0.643
    & 0.714 & 0.800 & 0.648 \\
    SenPred
    & 0.545 & 0.671 & 0.637
    & 0.537 & 0.677 & 0.627
    & 0.526 & 0.669 & 0.639 \\
    TextTiling
    & 0.558 & 0.561 & 0.526
    & 0.549 & 0.542 & 0.519
    & 0.545 & 0.562 & 0.520 \\
    TriBERT
    & 0.751 & 0.759 & 0.762
    & 0.759 & 0.760 & 0.749
    & 0.738 & 0.780 & 0.762 \\
    Ours
    & {0.834} & \textbf{0.864} & \textbf{0.858}
    & \textbf{0.842} & {0.877} & \textbf{0.851}
    & \textbf{0.849} & \textbf{0.873} & \textbf{0.848} \\
    \bottomrule
    \end{tabular}%
    }
\end{table}

\subsection{Runtime analysis}

We finally evaluate the computational cost of the proposed method on the synthetic documents considered in Section~\ref{sec:compare-baselines}. For each document, we record the runtime from the input text to the final token-level scores. We include both the fixed-width and adaptive-width variants with the triangular kernel. Figure~\ref{fig:run_time} shows that runtime increases approximately linearly with sequence length for both variants. The adaptive-bandwidth method takes about three seconds for a document with 1,300 tokens, roughly corresponding to 1,000 words or two pages of text. This implies that the proposal is capable of detecting long text. The fixed-width variant is faster than the Lepski-based adaptive variant, suggesting that it can be used as a computationally cheaper alternative when processing speed becomes a priority.

\begin{figure}[t]
\centering
\includegraphics[width=1.0\linewidth]{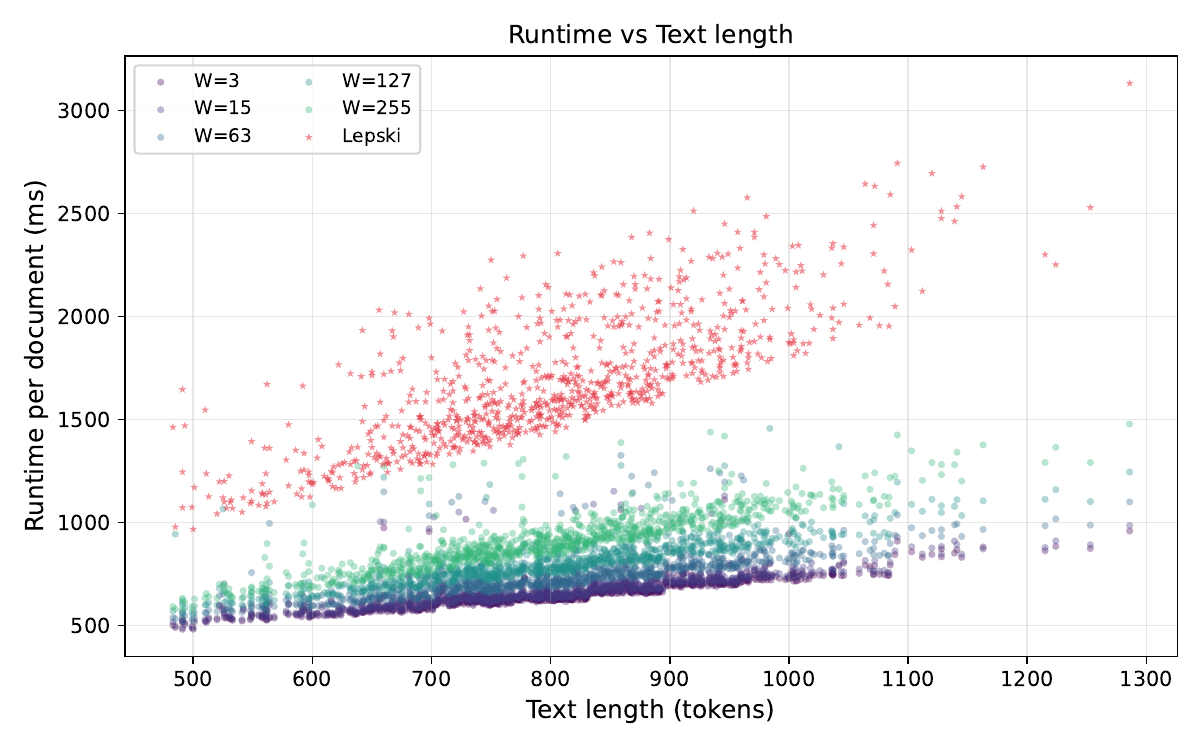}
\vspace{-12pt}
\caption{Runtime per document versus document length. Runtime is measured in milliseconds (ms).}
\label{fig:run_time}
\end{figure}

\section{Case study}\label{sec:case-study}

We deploy the proposed method as a publicly accessible website\footnote{\url{https://huggingface.co/spaces/Jasmineame/detect-LLM-tokens}} that localizes LLM-written tokens in human--LLM coauthored text and reports the estimated proportion of LLM-generated content. The interface is designed to make the token-level decisions returned by our proposal visually inspectable --- it highlights the places that are predicted to be LLM-generated, allowing users to examine where the transition between human and LLM writing occurs. 

Figure~\ref{fig:website_snapshot} shows a snapshot of the interface. To analyze a passage, the user enters text in the input box in the left-hand-side of the webpage, selects either a fixed bandwidth from the slider or the adaptive bandwidth\footnote{Adaptive bandwidth is the default option.}, and chooses the kernel type. The triangular kernel will be used by default, while the interface also allows the user to switch to the uniform kernel. After the user clicks the blue button ``Scan'', the website analyzes the text and returns a token-level visualization of the detection results in the ``Results'' box on the right-hand of the webpage.
\begin{figure}[t]
\centering
\includegraphics[width=1.0\linewidth]{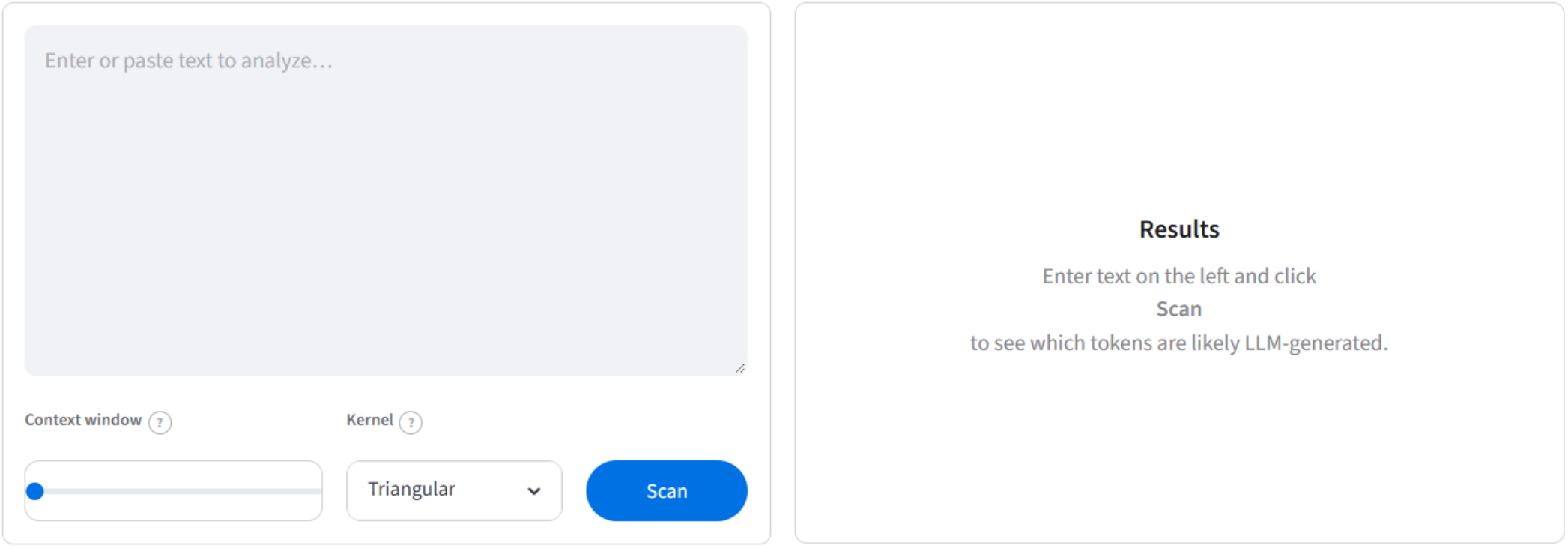}
\caption{A snapshot of the demonstration website.}
\label{fig:website_snapshot}
\end{figure}
\begin{figure}[t]
\centering
\includegraphics[width=1.0\linewidth]{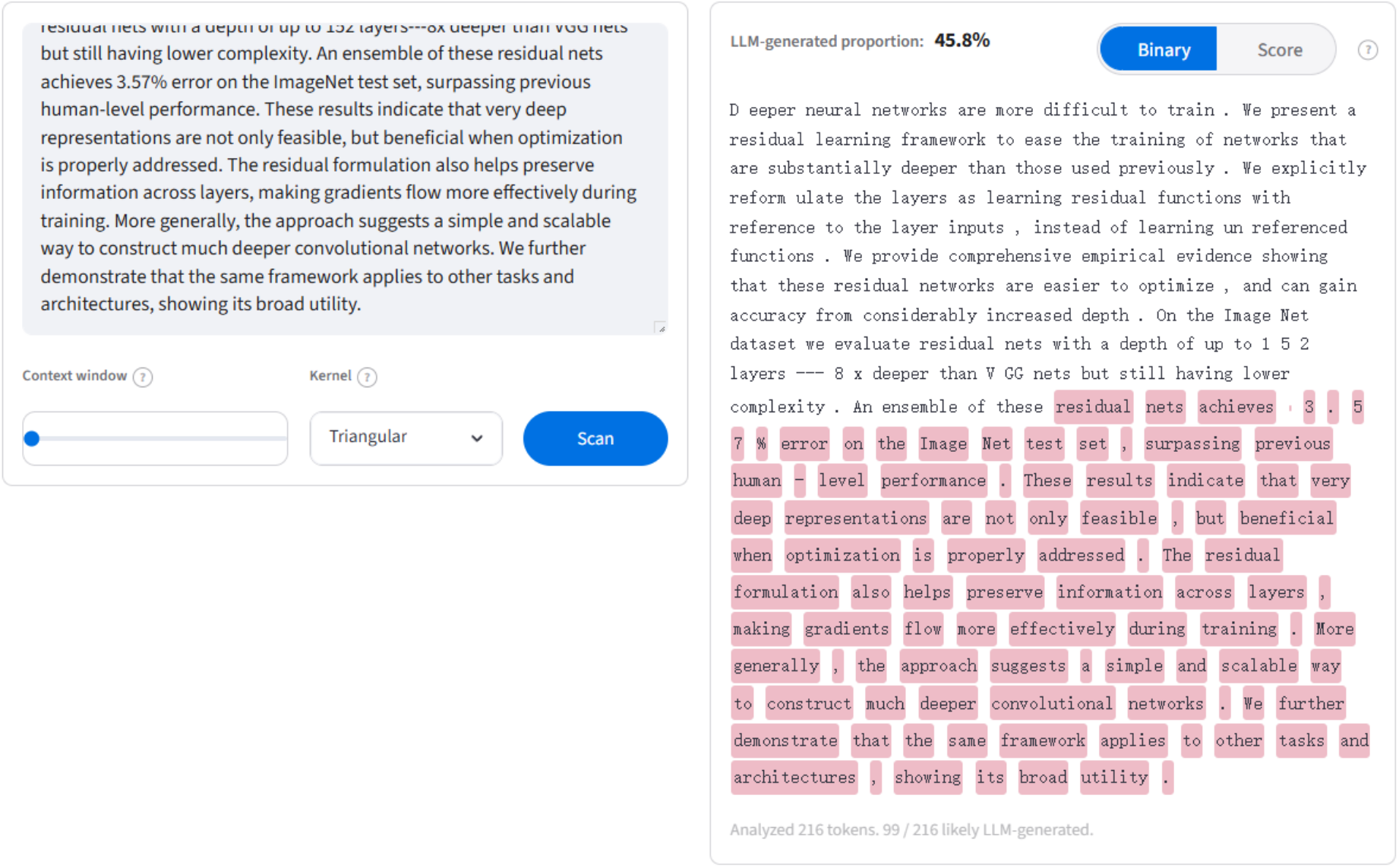}
\caption{The snapshot of the detection results for the human--LLM coauthored abstract based on \citet{he2016deep}.}
\label{fig:instance_rewrite}
\end{figure}

As an illustration, we consider the abstract of the paper ``Deep Residual Learning for Image Recognition'' \citep{he2016deep}, one of the most widely cited papers in computer science with total citation over 300 thousand. We first use this abstract to create a human-AI coauthored passage.
Specifically, we keep the first half of the original tokens in the abstract and ask GPT-5 \citep{openai2025gpt5card} to generate the remaining text. The detailed prompt and the GPT-5's output are provided in Appendix~\ref{sec:case-study-details}. We then input the resulting passage into the input box in the web interface. As shown in Figure~\ref{fig:instance_rewrite}, the highlighted tokens are highly consistent with the LLM-generated pieces. Indeed, 204 out of 216 tokens are correctly classified according to the classification results presented in Figure~\ref{box:case-classification}. While this single case should not be interpreted as a general performance measure, it shows the deployed tool can detect LLM-generated segments in a concrete coauthoring scenario.
\begin{figure}[H]
\begin{tcolorbox}[colback=white, colframe=black, colbacktitle=black, coltitle=white, fonttitle=\bfseries, title=Classification Results on Human and GPT-5 coauthored text]
\correcthuman{Deeper neural networks are more difficult to train. We present a residual learning framework to ease the training of networks that are substantially deeper than those used previously. We explicitly reformulate the layers as learning residual functions with reference to the layer inputs, instead of learning unreferenced functions. We provide comprehensive empirical evidence showing that these residual networks are easier to optimize, and can gain accuracy from considerably increased depth. On the ImageNet dataset we evaluate residual nets with a depth of up to 152 layers---8x deeper than VGG nets but still having lower complexity. An ensemble of these} \humantollm{residual nets achieves 3.57\% error on the} \correctllm{ImageNet test set, surpassing previous human-level performance. These results indicate that very deep representations are not only feasible, but beneficial when optimization is properly addressed. The residual formulation also helps preserve information across layers, making gradients flow more effectively during training. More generally, the approach suggests a simple and scalable way to construct much deeper convolutional networks. We further demonstrate that the same framework applies to other tasks and architectures, showing its broad utility.}
\end{tcolorbox}
\vspace{-12pt}
\caption{Detailed token-level classification for the human--LLM coauthored abstract. The 87 correctly classified LLM-generated tokens are highlighted in \correctllm{cyan}, and the 117 correctly classified human-written tokens are highlighted in \correcthuman{purple}. The 12 human-written tokens incorrectly classified as LLM-generated are highlighted in \humantollm{orange}. }\label{box:case-classification}
\end{figure}

We also consider an all-human control case by applying the detector directly to the original abstract of \citet{he2016deep}. Figure~\ref{fig:instance} shows that the method does not flag any tokens as LLM-generated, which is consistent with the realistic case.
\begin{figure}[t]
\centering
\includegraphics[width=1.0\linewidth]{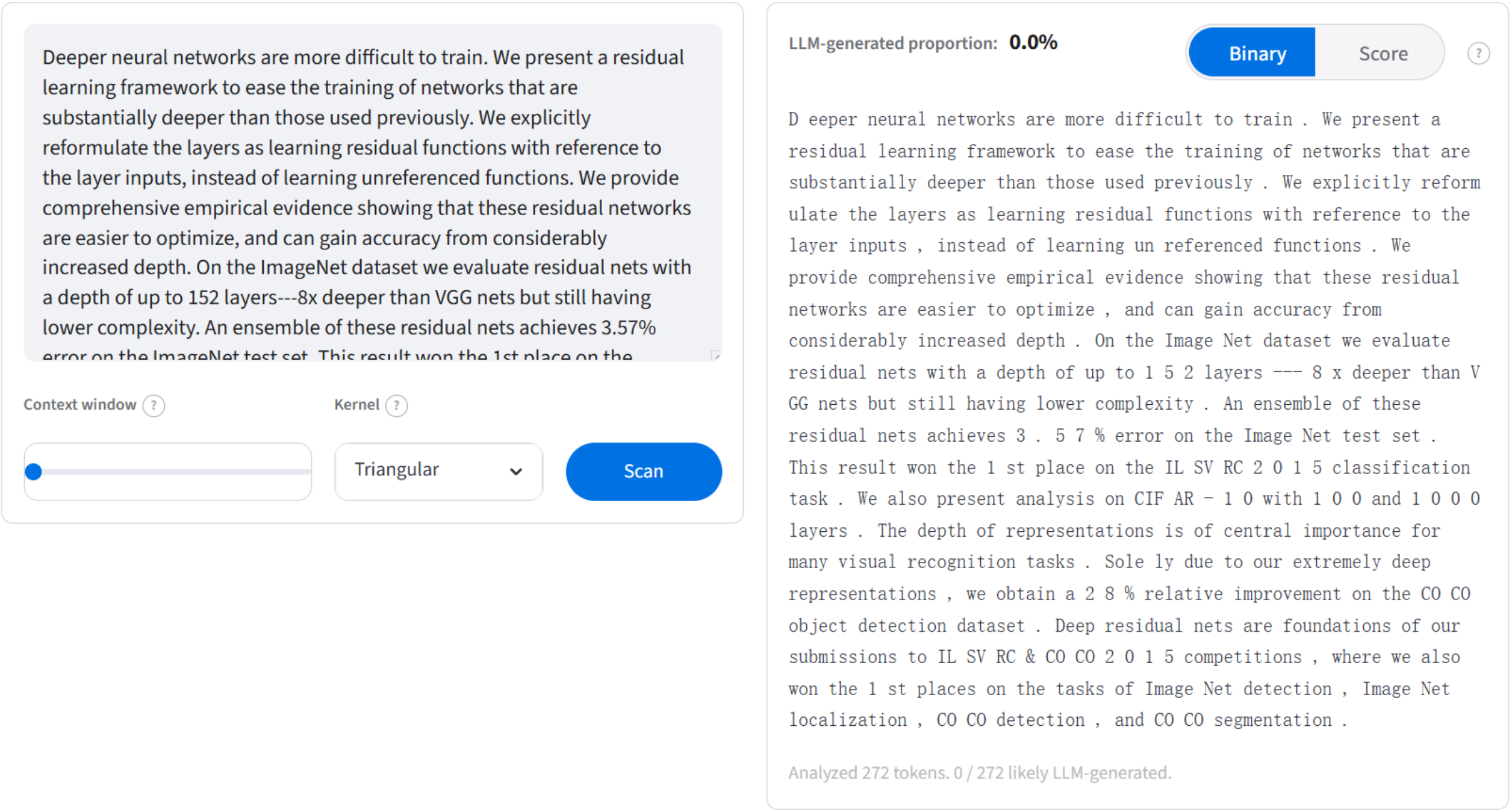}
\caption{The snapshot of the detection results for the original abstract from \citet{he2016deep}.}\label{fig:instance}
\end{figure}

\section{Conclusion}\label{sec:conclusion}

We studied token-level localization of LLM-generated content in human--LLM coauthored text. The proposed framework smooths token-level detection scores over neighboring tokens and uses a Lepski-type rule to adapt the bandwidth to local authorship structure. Our theory explains the bias--variance trade-off created by local aggregation, and our experiments show that adaptive aggregation improves performance across synthetic and realistic mixed-authorship settings. Future work may extend the framework to multilingual text, stronger adversarial editing, and richer forms of human--AI interaction.

\section*{Acknowledgments}
We would like to thank everyone who supported and encouraged us during the preparation of this paper.

Zhu's work was supported by the AIRR Gateway project 2026, ``Efficient Approaches on Detecting Generative AI''. Shi's and Ye's work was supported by the AIRR Gateway Project 2026, ``Reasoning-Enhanced Multimodal AI-Generated Content Detection''. Spill’s work was supported by a UKRI Future Leaders Fellowship (MR/T043571/1).

\newpage

\appendix
\section{Proof of Theorems}\label{sec:proof}

\subsection{Proof of Theorem \ref{thm:mse-general}}

\begin{proof}
Throughout the proof, all expectations, variances, covariances, and correlations are taken conditional on the fixed label sequence $C=(C_1,\ldots,C_L)$. For simplicity, we suppress this conditioning from the notation. Additionally, for a fixed token and a bandwidth , we omit the subscript $t$ and $k$ in $w_{t,l}^{(k)}$ for simplicity and denote $ w_l=w_{t,l}^{(k)}$.
By definition,
\[
\widehat S_{t,k}=\sum_{l\in\textup{Adj}_k(t)} w_l S_l \text{ with } \sum_{l\in \textup{Adj}_k(t)} w_l=1.
\]
According to the bias-variance decomposition, 
\begin{eqnarray}
    \text{MSE}(\widehat S_{t,k}) = \text{Bias}^2(\widehat S_{t,k}) +\text{Var}(\widehat S_{t,k}).
\end{eqnarray}
We first consider the bias term. Notice that
\begin{eqnarray}
    \text{Bias}(\widehat S_{t,k}) =\mathbb{E}[\widehat S_{t,k}]-\mu_{C_t}=\sum_{l\in \textup{Adj}_k(t)} w_l\mathbb{E}[S_l]-\mu_{C_t} = \sum_{l\in \textup{Adj}_k(t)}w_l(\mu_{C_l}-\mu_{C_t}).
\end{eqnarray}
By Assumption~\ref{assump:score-separable}, If $C_l\ne C_t$, then 
$|\mu_{C_l}-\mu_{C_t}|= |\mu_1-\mu_0| = \Delta$. Therefore,
\begin{eqnarray}
    \text{Bias}(\widehat S_{t,k}) = \Delta\cdot
\sum_{l\in \textup{Adj}_k(t):\,C_l\ne C_t}w_l=\Delta\cdot\pi_{t,k}.
\end{eqnarray}
We next bound the variance term. 
By the definition of $\widehat{S}_{t,k}$,
\[
\operatorname{Var}(\widehat{S}_{t,k})=
\operatorname{Var}\left(\sum_{l\in \textup{Adj}_k(t)}w_lS_l\right)
=
\sum_{l,m\in \textup{Adj}_k(t)}
w_lw_m\operatorname{Cov}(S_l,S_m).
\]
For any \(l,m\in \textup{Adj}_k(t)\), the definitions of \(\sigma^2\) and \(\rho(\cdot)\) imply
\[
|\operatorname{Cov}(S_l,S_m)|
=
|\operatorname{Corr}(S_l,S_m)|
\sqrt{\operatorname{Var}(S_l)\operatorname{Var}(S_m)}
\le
\sigma^2\rho(|l-m|).
\]
Combining the fact that the weights $w_l, l\in \textup{Adj}_k(t)$ are nonnegative, we obtain
\[
\operatorname{Var}(\widehat S_{t,k})
\le\sigma^2\sum_{l,m\in \textup{Adj}_k(t)}w_lw_m\rho(|l-m|)\leq \frac{1}{2}\sigma^2\sum_{l,m\in \textup{Adj}_k(t)}(w_l^2+w_m^2)\rho(|l-m|),
\]
where the last inequality follows from elementary inequality 
$2ab\le a^2+b^2$. Rearranging the term on the right hand side, we obtain 
\begin{eqnarray}
    \operatorname{Var}(\widehat S_{t,k})&\leq& \sigma^2\sum_{l\in \textup{Adj}_k(t)}w_l^2\sum_{m\in \textup{Adj}_k(t)}\rho(|l-m|)\nonumber\\
    &\le& \sigma^2 (n_{\mathrm{eff}}(t,k))^{-1} \left(1 + 2\sum_{i=1}^{2k}\rho(i)\right)\nonumber\\
    &\leq&2\kappa\sigma^2 (n_{\mathrm{eff}}(t,k))^{-1}.
\end{eqnarray}

Combining the bias and variance bounds, we obtain
\[
\operatorname{MSE}(\widehat S_{t,k})=
\text{Bias}^2(\widehat S_{t,k}) +\text{Var}(\widehat S_{t,k})
\le
\Delta^2\pi_{t,k}^2
+
\frac{2\kappa\sigma^2}{n_{\mathrm{eff}}(t,k)}.
\]
This finishes the proof.
\end{proof}

\subsection{Proof of Corollary \ref{cor:special-kernel}}
\begin{proof}
According to Theorem~\ref{thm:mse-general}, it remains to bound \(\pi_{t,k}\) and \(n_{\mathrm{eff}}(t,k)\) for each kernel.

\textbf{(i) Uniform kernel.} For the uniform kernel, it is easy to see that $w_{t,l}^{(k)}=\frac{1}{2k+1}$ for all $l\in\operatorname{Adj}_k(t)$.
If \(k\le d_t\), the whole window \(\operatorname{Adj}_k(t)\) is contained in the same constant segment as token \(t\). Hence
$\pi_{t,k}=0$.
If $k>d_t$, then the window crosses the nearer boundary. Since $k<\tau/2$, the window crosses at most one boundary. Thus, the number of tokens beyond that boundary and still inside the window is at most $k-d_t$. In this case,
\[
\pi_{t,k}\le\frac{k-d_t}{2k+1}.
\]
Combining both cases, we obtain
\[
\pi_{t,k}\le\frac{(k-d_t)_+}{2k+1},
\]
where \(x_+=\max\{x,0\}\).
Moreover, direct calculation yields that
\[
\sum_{l\in \operatorname{Adj}_k(t)}\left(w_{t,l}^{(k)}\right)^2
=(2k+1)\left(\frac{1}{2k+1}\right)^2=\frac{1}{2k+1}.
\]
Therefore, 
\[
n_{\mathrm{eff}}(t,k)=\left(\sum_{l\in \operatorname{Adj}_k(t)}\left(w_{t,l}^{(k)}\right)^2\right)^{-1}= 2k+1.
\]
The first conclusion of Corollary \ref{cor:special-kernel} then follows from substituting these two expressions into Theorem~\ref{thm:mse-general}.

\textbf{(ii) Triangular kernel.}
For the discrete triangular kernel, write \(r=l-t\). With some calculation, it is straightforward to show that the normalized weights are
\[
w_{t,t+r}^{(k)}=\frac{k+1-|r|}{(k+1)^2}, \quad r=-k,\ldots,k.
\]
We first bound the contamination weight. If \(k\le d_t\), then by the same argument, $\pi_{t,k}=0$. If \(k>d_t\), similar argument yields that the window $\operatorname{Adj}_k(t)$ crosses at most one boundary. Without loss of generality, suppose the nearest boundary is on the left side of \(t\). Then the contaminated indices correspond to distances $r=d_t+1,\ d_t+2,\ \ldots,\ k$ from $t$ on that side. Their total triangular weight is
\begin{eqnarray}
    \sum_{r=d_t+1}^{k}
\frac{k+1-r}{(k+1)^2} = \frac{1}{(k+1)^2}\sum_{s=1}^{k-d_t}s
=\frac{(k-d_t)(k-d_t+1)}{2(k+1)^2}.\nonumber
\end{eqnarray}
Therefore,
\[
\pi_{t,k}\le\frac{(k-d_t)(k-d_t+1)}{2(k+1)^2}.
\]
Combining the cases \(k\le d_t\) and \(k>d_t\), we get
\[
\pi_{t,k}\le \frac{(k-d_t)_+(k-d_t+1)_+}{2(k+1)^2}.
\]
Next, we compute the effective sample size. Notice that under the triangular kernel, 
\[
n_{\mathrm{eff}}(t,k)^{-1} = \sum_{l\in \operatorname{Adj}_k(t)}
\left(w_{t,l}^{(k)}\right)^2=\sum_{r=-k}^{k}\left(\frac{k+1-|r|}{(k+1)^2}\right)^2.
\]
Using the fact that
\[
\sum_{r=1}^{k}(k+1-r)^2=\sum_{s=1}^{k}s^2=\frac{k(k+1)(2k+1)}{6}.
\]
It follows that
\[
\sum_{r=-k}^{k}(k+1-|r|)^2=(k+1)^2+\frac{k(k+1)(2k+1)}{3}.
\]
Thus,
\[
n_{\mathrm{eff}}(t,k)^{-1} = \frac{2(k+1)^2+1}{3(k+1)^3}.
\]
The second conclusion of Corollary \ref{cor:special-kernel} then follows from substituting these two expressions into Theorem~\ref{thm:mse-general}.
\end{proof}

\subsection{Proof of Theorem \ref{thm:oracle-bound}}
\begin{proof}
We prove the result by verifying the assumptions required for
Corollary 4 of \citet{su2020adaptive}. For a fixed $t$-th token and a fixed kernel function $K$, define 
\begin{eqnarray}
    \mathfrak{B}_{t,k_i} := \Delta \pi_{t,k_i} = \Delta \sum_{l\in\text{Adj}_k(t):C_l\neq C_t}w_{t,l}^{(k_i)}.
\end{eqnarray}
To invoke Corollary 4 in \citet{su2020adaptive}, we need to verify that
\begin{itemize}[leftmargin=*]
    \item[(i)] $|\mathbb{E}[\widehat{S}_{t,k_i}]-\mu_{C_t}| \leq \mathfrak{B}_{t,k_i}$ for all $i$.
    \item[(ii)] With probability at least $1-\delta$, $|\widehat{S}_{t,k_i} - \mathbb{E}[\widehat{S}_{t,k_i}]| \leq\mathfrak{r}_{t,k_i}$ for all $i$.
    \item[(iii)] $\mathfrak{B}_{t,i}\leq\mathfrak{B}_{t,i+1}$ and $\kappa \mathfrak{r}_{t,i}\leq\mathfrak{r}_{t,i+1} \leq \mathfrak{r}_{t,i}$ for some constant $\kappa>0$ and for all $i$.
\end{itemize}

Condition (i) directly follows from a similar argument as deriving the bias in Theorem \ref{thm:mse-general}.

For condition (ii), define 
\[
Z_{t,k_i}:=\widehat S_{t,k_i}-\mathbb E[\widehat S_{t,k_i}]=
\sum_{l\in\operatorname{Adj}_{k_i}(t)}w_{t,l}^{(k_i)}\varepsilon_l.
\]
Under the assumption that the token scores are conditionally
independent and bounded by $\xi$,  we have $w_{t,l}^{(k_i)}\varepsilon_l\in\left[-\xi w_{t,l}^{(k_i)},\xi w_{t,l}^{(k_i)}\right]$ almost surely. Applying Hoeffding's inequality yields
\[
\mathbb P\left(|Z_{t,k_i}|\ge x\right)
\le2\exp\left(-\frac{x^2}{2\xi^2\sum_{l\in\operatorname{Adj}_{k_i}(t)}\left(w_{t,l}^{(k_i)}\right)^2}\right)= 2\exp\left(-\frac{x^2}{2\xi^2 n_{\mathrm{eff}}^{-1}(t,k_i)}\right).
\]
Now take
\[
\mathfrak r_{t,k_i}:=\left\{2\xi^2 n_{\mathrm{eff}}^{-1}(t,k_i)\log(2M/\delta)\right\}^{1/2}.
\]
Then
\[
\mathbb P\left(|Z_{t,k_i}|\ge \mathfrak r_{t,k_i}\right)\le2\exp\{-\log(2M/\delta)\}=\frac{\delta}{M}.
\]
Taking a union bound over \(i=1,\ldots,M\), we have
\[
\mathbb P\left(\forall i\in\{1,\ldots,M\},
\quad\left|\widehat S_{t,k_i}
-\mathbb E[\widehat S_{t,k_i}]\right|\le\mathfrak r_{t,k_i}\right)
\ge1-\delta.
\]
This verifies condition (ii).

For condition (iii), the monotonicity of the bias term directly follows from the proof of Corollary \ref{cor:special-kernel}. Moreover,
notice that the radius
\[
\mathfrak r_{t,k_i}=\sqrt{2}\xi\left\{n_{\mathrm{eff}}^{-1}(t,k_i)
\log(2M/\delta)\right\}^{1/2}.
\]
Following the same line as the proof in Corollary \ref{cor:special-kernel}, we obtain that $n_{\text{eff}}(t,k)$ is nonincreasing with respect to $k$ for both uniform and triangular kernel, thus $\mathfrak r_{t,k_{i+1}}\le \mathfrak r_{t,k_{i}}$. Meanwhile, based on the exact form of $n_{\text{eff}}(t,k)$ of uniform and triangular kernel and the geometric property of the grid (i.e. $k_{i+1}  = 2 k_i$), we obtain $3n_{\text{eff}}(t,k_{i+1})\ge n_{\text{eff}}(t,k_{i})$ for both kernels. Thus condition (iii) holds when we set $\kappa$ to $\kappa_0 = 1/3$.

With all the prerequisite verified, applying Corollary~4 of \citet{su2020adaptive} with $\widehat\theta_i=\widehat S_{t,k_i}, \theta^\star=\mu_{C_t}, B(i)=\mathfrak{B}_{t,k_i}$, $\operatorname{CNF}(i;\delta)=\mathfrak r_{t,k_i}$ and notice that $|S_t| \leq \xi+\max\{|\mu_0|,|\mu_1|\} =: R_S$, we obtain for a universal constant $C>0$,
\[
\mathbb E\left[\left(\widehat S_{t,k_{\widehat i_t}}-\mu_{C_t}\right)^2\right]
\le\frac{C}{\kappa_0^2}\min_{1\le i\le M}
\left[\mathfrak B_{t,k_i}^{\,2}+\mathfrak r_{t,k_i}^{\,2}\right]+ 4R_S^2\delta.
\]
Denote $C_{\text{Lep}} = 2C/\kappa_0^2$ and plug in the exact form of $\mathfrak B_{t,k_i}$ and $\mathfrak r_{t,k_i}$, we get
\[
\mathbb E\left[\left(\widehat S_{t,k_{\widehat i_t}}-\mu_{C_t}\right)^2\right]
\le\frac{C}{\kappa_0^2}\min_{1\le i\le M}\left[\Delta^2\left(\max_{1\le j\le i}\pi_{t,k_j}\right)^2
+\frac{\xi^2\log(2M/\delta)}{n_{\mathrm{eff}}(t,k_i)}\right]
+ 4R_S^2\delta.
\]
This proves the theorem.
\end{proof}

\section{Details on Implementation and Experiments}\label{app:implementation}

\subsection{Implementation Details on Adaptive Bandwidth Selection}
Recall that, for each candidate bandwidth, we compute a score estimator
$\widehat{S}_{t,k_i}$ and construct an interval
\[
I_t(i)=\left[\widehat S_{t,k_i}-2\widehat{\mathfrak r}_{t,k_i}(\delta),\,
\widehat S_{t,k_i}+2\widehat{\mathfrak r}_{t,k_i}(\delta)\right],
\]
where the interval radius $\widehat{\mathfrak r}_{t,k_i}(\delta)$ can be chosen as
\[
\widehat{\mathfrak r}_{t,k_i}(\delta)=\widehat{\xi}
\sqrt{\log(2M/\delta)\sum_{l\in \operatorname{Adj}_{k_i}(t)}\left(w_{t,l}^{(k_i)}\right)^2
}.
\]
Therefore, the Lepski rule can be written equivalently as
\begin{align*}
\max_{i\in\{1,\ldots,M\}}\Bigg\{i:
\bigcap_{j=1}^{i}
\Bigg[
&\frac{\widehat S_{t,k_i}}{\widehat\xi}
-2\sqrt{\log(2M/\delta)
\sum_{l\in \operatorname{Adj}_{k_i}(t)}\left(w_{t,l}^{(k_i)}\right)^2},\\
&\frac{\widehat S_{t,k_i}}{\widehat\xi}
+2\sqrt{\log(2M/\delta)
\sum_{l\in \operatorname{Adj}_{k_i}(t)}\left(w_{t,l}^{(k_i)}\right)^2}
\Bigg]\ne \emptyset\Bigg\}.
\end{align*}
In our implementation, the raw score for the $t$-th token can be instantiated as
\[
\widehat{S}_t = \log p_\theta(X_t\mid \bm{X}_{<t}).
\]
Therefore, a natural variance estimator for the raw token score is
\[
\mathbb{V}_{\widetilde{X}_t\sim s_\theta(\cdot\mid \bm{X}_{<t})}
\left[\log p_\theta(\widetilde{X}_t\mid \bm{X}_{<t})\right].
\]
To stabilize the variance estimator, we aggregate the variance estimates over the neighborhood $\textup{Adj}_k(t)$. Thus, the local variance estimator can be written as
\begin{equation*}
    \widehat{\xi}_{t,k_i}^2=
    \frac{\sum_{l\in\textup{Adj}_{k_i}(t)}\left(w_{t,l}^{(k_i)}\right)^2
    \mathbb{V}_{\widetilde{X}_l\sim s_\theta(\cdot\mid X_{<l})}
    \left[\log p_\theta(\widetilde{X}_l\mid X_{<l})\right]}
    {\sum_{l\in\textup{Adj}_{k_i}(t)}\left(w_{t,l}^{(k_i)}\right)^2}.
\end{equation*}

\subsection{Experiments: Details}\label{sec:experiments-details}
\paragraph{Prompt template.} We use the following template to generate LLM text in all experiments. For different authorship-transition patterns, we only adjust \texttt{input\_text} and \texttt{target\_sentence\_count}.
\begin{center}
\begin{tcolorbox}[colback=white, colframe=black, colbacktitle=black, coltitle=white, fonttitle=\bfseries, title=Prompt template]
[System]

None

[User]

Based on the given passage, write a natural continuation. The continuation should be coherent with the passage.

A continuation of about \{target\_sentence\_count\} sentences is preferred, but it may be slightly longer if needed for coherence.

Do not add explanation, commentary, numbering, or any prefix. Only output the continuation itself.

Passage: \{input\_text\}

Continuation:
\end{tcolorbox}
\end{center}

\paragraph{Data construction: truncating overly long documents.} When constructing the synthetic datasets, we truncate overly long passage to let it have at most 32 sentences. This truncation serves two purposes: (i) it reduces the computational cost of synthetic data generation and (ii) it facilitates fair comparison with DeBERTa, Segformer and TriBERT, whose maximum input lengths differ because of tokenizer constraints. Under this setting, performance differences among detectors remain observable.

\paragraph{Model details.}
We use six LLMs to generate human-AI coauthored text. Ranging from open-sourced Gemma-2B-Instruct model and close-sourced models such as Grok-4.1-Fast-non-reasoning, Claude-4.5-haiku, Gemini-3.1-Flash-Lite, GPT-5 and GPT-5-nano. We deploy Gemma-2B-Instruct locally, while access the remaining models through their APIs. For the locally deployed Gemma model, we use greedy decoding with a fixed random seed (\texttt{seed=0}). For API-based models, we retain the default generation parameters. In the sensitivity analysis, we only vary the decoding temperature over \(\{0.3, 0.6, 0.9\}\) while keeping all other generation parameters fixed.

\paragraph{Baselines: implementation details.} We summarize the training procedure for each baseline below. For all ML-based methods, we construct an additional dataset $\mathcal{D}$ using the same procedure as in Section~\ref{sec:compare-baselines}, with three changes: (i) $q$ is fixed at 1; (ii) the human text is sourced from the Yelp dataset \citep{zhangCharacterlevelConvolutionalNetworks2015}; and (iii) the LLM-generated tokens are produced by GPT-5-nano. This setup mimics the realistic cases in which the training distribution differs from the test distribution. For the baseline methods that do not implement training/validation pipeline, we use an 80/20 train--validation split.

We note that when the authors of these baselines publicly release the training code, we use the official implementation for training whenever possible. Unless otherwise specified, all ML-based methods are optimized with AdamW using linear learning-rate decay and zero weight decay. For the method that do not specify learning rate and batch size, we set the learning rate to 5e-5 and the batch size to 2. 
\begin{enumerate}[leftmargin=*]
\item \textbf{DeBERTa.} We train DeBERTa on $\mathcal{D}$ and set the training hyperparameters following \citet{su2025haco}. For long documents, due to the maximum input constraint (i.e., 512 tokens) in DeBERTa's implementation, we split the input into overlapping windows of at most 512 tokens, with an overlap of 128 tokens between adjacent windows.
\item \textbf{SegFormer.} SegFormer is a ML-based method. We use the training implementation released by the authors (\url{https://github.com/douglashiwo/AISentenceDetection}) and train the model on our dataset. For most of the remaining hyperparameters in their implementation, we use the default values. 
\item \textbf{TriBERT.} We train TriBERT following the procedure described in \citet{zeng2024towards} and use the official implementation released at \url{https://github.com/douglashiwo/BoundaryDetection}. We adopt all of the their hyperparameter setting in their implementation such as a proportion of the training/validation set. 
\item \textbf{SenPred.} Following the pipeline of \cite{kushnareva2024boundary} provided at \url{https://github.com/SilverSolver/ai_boundary_detection}, SenPred is trained and valided on $\mathcal{D}$ using sentence-level feature vectors extracted from the text. Each sentence is quantitatively represented by the mean and standard deviation of its token-level scores, which are used as features in logistic regression.
\item \textbf{TextTiling.} TextTiling is a training-free segmentation method, so no supervised training is required. We use the \texttt{TextTilingTokenizer} function in the \texttt{nltk} Python library \citep{bird2009natural}. The only tuning parameter, the block size, is set to its default value of 20.
\item \textbf{PaLD.} We implement the greedy algorithm of PaLD introduced in \citet{lei2025pald}, whose time complexity grows quadratically with respect to the number of sentences. 
\end{enumerate}

\paragraph{Hardware setting.} All experiments are run on a platform with a Tesla V100 GPU with 32GB of memory.

\subsection{Case studies: Details}\label{sec:case-study-details}
The prompt used to ask GPT-5 to continue the abstract of \citet{he2016deep} is shown in Figure~\ref{fig:prompt-gpt5} and the corresponding output is provided in Figure~\ref{fig:response-gpt5}.
\begin{figure}[t]
\begin{tcolorbox}[colback=white, colframe=black, colbacktitle=black, coltitle=white, fonttitle=\bfseries, title=Prompt to GPT-5]
Based on the given passage, write a natural continuation. The continuation should be coherent with the passage. A continuation of about 6 sentences is preferred, but it may be slightly longer if needed for coherence. Do not add explanation, commentary, numbering, or any prefix. Only output the continuation itself.

Passage: Deeper neural networks are more difficult to train. We present a residual learning framework to ease the training of networks that are substantially deeper than those used previously. We explicitly reformulate the layers as learning residual functions with reference to the layer inputs, instead of learning unreferenced functions. We provide comprehensive empirical evidence showing that these residual networks are easier to optimize, and can gain accuracy from considerably increased depth. On the ImageNet dataset we evaluate residual nets with a depth of up to 152 layers---8x deeper than VGG nets but still having lower complexity. An ensemble of these residual nets achieves 3.57\% error on the

Continuation:
\end{tcolorbox}
\vspace*{-12pt}
\caption{}\label{fig:prompt-gpt5}
\end{figure}

\begin{figure}[t]
\begin{tcolorbox}[colback=white, colframe=black, colbacktitle=black, coltitle=white, fonttitle=\bfseries, title=Response of GPT-5]
ImageNet test set, surpassing previous human-level performance. These results indicate that very deep representations are not only feasible, but beneficial when optimization is properly addressed. The residual formulation also helps preserve information across layers, making gradients flow more effectively during training. More generally, the approach suggests a simple and scalable way to construct much deeper convolutional networks. We further demonstrate that the same framework applies to other tasks and architectures, showing its broad utility.
\end{tcolorbox}
\vspace*{-12pt}
\caption{}\label{fig:response-gpt5}
\end{figure}

\bibliographystyle{unsrtnat}
\bibliography{reference}

\end{document}